\documentclass[accepted]{uai2022}

\usepackage[american]{babel}

\usepackage{natbib} 
\usepackage{mathtools} 
\usepackage{booktabs} 
\usepackage{tikz} 

\usepackage{amsmath,amsfonts,amssymb,mathtools, amsthm, dsfont}
\usepackage{subfiles}
\usepackage{tikz, subcaption}
\usepackage[noend]{algorithmic}
\usepackage[ruled,vlined,linesnumbered]{algorithm2e}
\usepackage{bbm}

\newtheorem{theorem}{Theorem}

\newtheorem{lemma}{Lemma}

\newenvironment{customlem}[1]{\renewcommand\thelemma{#1}\lemma}{\endlemma}

\newtheorem{proposition}{Proposition}
\newenvironment{customprp}[1]{\renewcommand\theproposition{#1}\proposition}{\endproposition}
\newtheorem{definition}{Definition}
\newtheorem{remark}{Remark}
\newtheorem*{lemma*}{Lemma}

\newenvironment{myproof}[1][\proofname]{%
  \begin{proof}[#1]$ $\nobreak\ignorespaces
}{%
  \end{proof}
}

\usetikzlibrary{positioning}

\newcommand{\Pa}[2]{\textit{Pa}_{#2}(#1)}
\newcommand{\Ch}[2]{\textit{Ch}_{#2}(#1)}
\newcommand{\Anc}[2]{\textit{Anc}_{#2}(#1)}
\newcommand{\V}[0]{\mathbf{V}}
\newcommand{\C}[0]{\mathbf{C}}
\newcommand{\W}[0]{\mathbf{W}}
\newcommand{\U}[0]{\mathbf{U}}
\newcommand{\E}[0]{\mathbf{E}}
\newcommand{\X}[0]{\mathbf{X}}
\newcommand{\R}[0]{\mathbf{R}}
\newcommand{\Y}[0]{\mathbf{Y}}
\newcommand{\Z}[0]{\mathbf{Z}}
\newcommand{\A}[0]{\mathbf{A}}
\newcommand{\B}[0]{\mathbf{B}}
\newcommand{\T}[0]{\mathbf{T}}
\newcommand{\x}[0]{\mathbf{x}}
\newcommand{\y}[0]{\mathbf{y}}
\newcommand{\z}[0]{\mathbf{z}}
\newcommand{\G}[0]{\mathcal{G}}
\newcommand{\M}[0]{\mathcal{M}}
\newcommand{\F}[0]{\mathcal{F}}
\newcommand{\dom}[1]{\mathfrak{X}_{#1}}

\title{Revisiting the General Identifiability Problem}

\author[1]{Yaroslav Kivva}
\author[1]{Ehsan Mokhtarian}
\author[1]{Jalal Etesami}
\author[1,2]{Negar Kiyavash}
\affil[1]{%
    School of Computer and Communication Sciences\\
    EPFL\\
    Lausanne, Switzerland
}
\affil[2]{%
    College of Management of Technology\\
    EPFL\\
    Lausanne, Switzerland
}

\begin{document}
\maketitle

\begin{abstract}
    We revisit the problem of general identifiability originally introduced in \citep{lee2019general} for causal inference and note that it is necessary to add positivity assumption of observational distribution to the original definition of the problem. We show that without such an assumption the rules of do-calculus and consequently the proposed algorithm in \citep{lee2019general} are not sound. Moreover, adding the assumption will cause the completeness proof in \citep{lee2019general} to fail. Under positivity assumption, we present a new algorithm that is provably both sound and complete. A nice property of this new algorithm is that it establishes a connection between  general identifiability and classical identifiability by \cite{pearl1995causal} through decomposing the general identifiability problem into a series of classical identifiability  sub-problems.
\end{abstract}

\section{Introduction}
    Causal effect identification (or {ID} for short) problem, a central concern in causal inference, pertains to whether, given a causal graph, an interventional distribution can be uniquely computed from observational distribution \citep{pearl2009causality}.
    When all the variables in the system are observable, Pearl's do-calculus (a collection of three rules) allows determining whether a causal effect is identifiable  \citep{pearl1995causal}.
    Moreover, it was shown that Pearl's do-calculus is both sound and complete for ID problem \citep{shpitser2006identification, huang2008completeness}.
    
    In the classical setting of ID problem, both the causal graph and the observational distribution, denoted by $P(\V)$ ($\V$ is the set of observed variables in the causal graph), are given.
    However, it is assumed that no extra information (such as interventional distribution) is available.
    Recently, several work in the literature relax these assumptions \citep{tikka2019causal,shpitser2012identification,bareinboim2015recovering,bareinboim2014transportability,mokhtarian2022causal}.
    Before discussing these results, let us introduce a notion. 
    We denote by $P_{\x}(\Y)$ the distribution of a set of variables $\Y$ resulting from intervening on another set of variables $\X$.
    \cite{bareinboim2012causal} introduced the z-identification problem (or zID for short) in which for a fixed set $\Z \subseteq \V$, given a set of interventional distributions of the form $\{P_{\z'}(\V):\ \forall \Z' \subseteq \Z\}$, one asks whether $P_{\x}(\Y)$ is identifiable. 
    Note that the observational distribution $P(\V)$ always belongs to the set of available distributions.
    Furthermore, the form of given interventional distributions is restrictive.
    \cite{lee2019general} generalized zID and proposed so-called general identifiability problem (or gID for short). 
    In the gID, observational distribution is not necessarily given but instead we have access to $\{P_{\z_i}(\V)\}_{i=0}^{m}$ for some subsets $\{\Z_i\}_{i=0}^{m}$ of observed variables.  
    When one of $\Z_i$s is an empty set, we have access to $P(\V)$.
    
    We give formal definitions of identifiability (Definition \ref{def: id}) and general identifiability (Definition \ref{def: gid}) in Section \ref{sec: pre}.
    An important contribution of this paper is to add an assumption on the positivity of the observational distribution in the definition of general identifiability, i.e., $P(\mathbf{v})>0$ for all the realizations of observed variables.
    As we shall discuss in detail in Section \ref{sec: positivity}, this assumption, or at least a relaxed version of it, is crucial.
    More specifically, do-calculus-based methods are no longer sound for the ID problem if we ignore the positivity assumption. 
    In other words, there exist causal graphs with non-positive distribution $P(\V)$ such that do-calculus would claim a causal effect is identifiable while it cannot be uniquely computed from mere observational distribution.
    Violation of the positivity assumption can happen in practice. For instance, some empirical distributions would be zero when the observational data is not large enough.
    An even more important reason for including the positivity assumption is that without it, the proposed algorithm in the original gID in \citep{lee2019general} is not sound.
    Furthermore, as we shall discuss in Section \ref{sec: positivity}, the proof of completeness in \citep{lee2019general} relies on building two models that have zero probabilities for certain realizations of observed variables.
    Therefore, unfortunately, simply adding the positivity assumption to the definition of general identifiability (g-identifiability) will fail the proof technique in \citep{lee2019general} for the completeness of their proposed algorithm.
    On the other hand, ignoring the positivity assumption makes the soundness of their algorithm incorrect.
    
    In summary, our main contributions are as follows.
    We redefine the g-identifiability by adding the positivity assumption of observational distribution (Definition \ref{def: gid}).  
    We show in Section \ref{sec: positivity} that this assumption is essential for the gID problem.
    We then provide a sound and complete algorithm for the gID problem (Algorithm \ref{algo: GID}). 
    A nice property of our algorithm is that it establishes a connection between gID and classical ID by showing that gID can be reduced to solving a series of ID problems (Theorem \ref{thm: main}).

\section{Preliminaries} \label{sec: pre}
    \subsection{Terminology}
        Throughout the paper, we denote random variables by capital letters (e.g., $X$), their realizations by small letters (e.g., $x$), and sets by bold letters (e.g., $\X$ or $\x$). 
        We use $\dom{X}$ to denote the domain of random variable $X$ and $\dom{\X}$ to denote the Cartesian product of the domains of all the variables in set $\X$, i.e., $\prod_{X\in \X} \dom{X}$. 
        For integer numbers $a\leq b$, we use $[a:b]$ to denote $\{a,a+1,\cdots,b\}$.
        
        Suppose $\G=(\V \cup \U, \E)$ is a directed acyclic graph (DAG) over vertex set $\V \cup \U$, where $\V$ and $\U$ represent the set of observed and unobserved variables, respectively. 
        For each edge $(X, Y)\in \E$, $X$ is called a parent of $Y$, and $Y$ is called a child of $X$. 
        Vertex $X$ is an ancestor of $Y$ in $\G$ if a directed path exists from $X$ to $Y$ in $\G$.
        Note that $X$ is an ancestor of itself. 
        $\Pa{X}{\G}$, $\Ch{X}{\G}$, and $\Anc{X}{\G}$ denote the set of parents, children, and ancestors of $X$ in $\G$, respectively.
        These notations are also used for a set of vertices. 
        In this case, they refer to the union over the set elements. 
        For instance,  $\Pa{\X}{\G}=\bigcup_{X \in \X}\Pa{X}{\G}$.
        We assume $\G$ is semi-Markovian, that is for each $U \in \U$, $\Pa{U}{\G}= \varnothing$ and $|\Ch{U}{\G}|=2$. 
        Note that this is not a restrictive assumption as there exists an equivalency for identifiability in DAGs and semi-Markovian DAGs \citep{huang2006identifiability}.
        
        Structural Equation Models (SEMs) are used to model causal systems \citep{pearl2009causality}. 
        $\G$ is a causal graph for SEM $\mathcal{M}$ if each $X\!\in\! \V \cup \U$ is generated as $f_X(\Pa{X}{\G}, \epsilon_X)$, where $\{\epsilon_X\!: X\in \V\}$ is a set of mutually independent exogenous random variables.
        We denote by $P^{\M}(\cdot)$ the joint distribution of the variables in $\M$ and drop the superscript $\M$ when it is clear from the context.
        Markov factorization property implies that $P^{\M}(\cdot)$ can get factorized as
        \begin{equation} \label{eq: factorization 1}
            P(\mathbf{v})= \sum_{\U} \prod_{X \in \V}P(x|\Pa{X}{\G}) \prod_{U\in \U}P(u),
        \end{equation}
        where $\sum_{\U}$ denotes the marginalization over $\U$.
        \begin{definition}
            $\mathbb{M}(\G)$ denotes the set of SEMs with causal graph $\G$. 
            $\mathbb{M}^{+}(\G)$ denotes the set of SEMs $\M \in \mathbb{M}(\G)$  such that $P^{\M}(\mathbf{v})>0$ for each $\mathbf{v} \in \dom{\V}$.
        \end{definition}
        
        For $\X \subseteq \V$ and $\x \in \dom{\X}$, the intervention $do(\X = \x)$ converts $\M$ to a new SEM where the equations of $\X$ in $\M$ are replaced by the constants in $\x$. 
        We denote by $P_{\x}(\cdot)$ the corresponding post interventional distribution.
        \begin{remark}
            For three disjoint subsets $\X, \Y, \mathbf{W}$ of $\V$, if $\M \in \mathbb{M}^+(\G)$, then $P^{\M}_{\x}(\y \mid \mathbf{w})>0$
            for any $\x \in \dom{\X}$, $\y \in \dom{\Y}$, and $\mathbf{w} \in \dom{\mathbf{W}}$.
        \end{remark}
        For $\mathbf{v} \in \dom{\V}$ and $\mathbf{S} \subseteq \V$, we define $Q[\mathbf{S}](\cdot)$ by 
        \begin{equation} \label{eq: Q}
            Q[\mathbf{S}](\mathbf{v}) := P_{\mathbf{v}\setminus \mathbf{s}}(\mathbf{s}).
        \end{equation}
        Similar to Equation \eqref{eq: factorization 1}, $Q[\mathbf{S}]$ can get factorized as 
        \begin{equation} \label{eq: factorization 2}
            Q[\mathbf{S}](\mathbf{v})= \sum_{\U} \prod_{S \in \mathbf{S}}P(s|\Pa{S}{\G}) \prod_{U\in \U}P(u).
        \end{equation}

        For $\X \subseteq \V$, $\G[\X]$ denotes the inducing subgraph of $\G$ over $\X$ and the unobserved variables with both children in $\X$. 
        Note that $\G$ is semi-Markovian.
        Furthermore, we denote by $\G_{\X}$ the partially directed graph over $\X$ obtained by removing unobserved variables of $\G[\X]$ and replacing them by bidirected edges.
    
        \begin{definition}[\textbf{c-component, c-forest}] \label{def: c-component and c-forest}
            For $\X \subseteq \V$, confounded components or c-components of $\X$ are the connected components of the graph obtained by only the bidirected edges of $\G_{\X}$.
            Also, a subgraph of $\G_\V$ is called a single c-component if its bidirected edges form a connected graph. 
            Suppose $\mathcal{H}$ is a subgraph of $\G$ over observed vertices $\X$. 
            The root set of $\mathcal{H}$ is the maximal subset of $\X$ with no children in $\mathcal{H}$. 
            $\mathcal{H}$ is called $\mathbf{R}$-rooted c-forest if $\mathbf{R}$ is the root set of $\mathcal{H}$, $\mathcal{H}_{\X}$ is a single c-component, and each node in $\X$ has at most one child in $\mathcal{H}$. 
        \end{definition}
        
        \begin{figure}
            \centering
            \begin{subfigure}[b]{.23\textwidth}
                \centering
                \begin{tikzpicture}[
                roundnode/.style={circle, draw=black!60,, fill=white, thick, inner sep=1pt},
                dashednode/.style = {circle, draw=black!60, dashed, fill=white, thick, inner sep=1pt},
                ]
                    \node[roundnode]        (X1)        at (0, 0)                   {$X_1$};
                    \node[roundnode]        (X2)        at (2, 0)                   {$X_2$};
                    \node[roundnode]        (Y1)        at (0, -2)                  {$Y_1$};
                    \node[roundnode]        (Y2)        at (2, -2)                  {$Y_2$};
                    \node[dashednode]       (U1)        at (1, 0)                   {$U_1$};
                    \node[dashednode]       (U2)        at (1, -2)                   {$U_2$};
                    
                    \draw[-latex] (X1.south) -- (Y1.north) ;
                    \draw[-latex] (X2) -- (Y2);
                    \draw[latex-, dashed] (X1) -- (U1);
                    \draw[-latex, dashed] (U1) -- (X2);
                    \draw[latex-, dashed] (Y1) -- (U2);
                    \draw[-latex, dashed] (U2) -- (Y2);
                \end{tikzpicture}
                \caption{$\G$}
                \label{fig: G}
            \end{subfigure}
            \hfill
            \begin{subfigure}[b]{.23\textwidth}
                \centering
                \begin{tikzpicture}[
                roundnode/.style={circle, draw=black!60,, fill=white, thick, inner sep=1pt},
                dashednode/.style = {circle, draw=black!60, dashed, fill=white, thick, inner sep=1pt},
                ]
                    \node[roundnode]        (X1)        at (0, 0)                   {$X_1$};
                    \node[roundnode]        (X2)        at (2, 0)                   {$X_2$};
                    \node[roundnode]        (Y1)        at (0, -2)                  {$Y_1$};
                    \node[roundnode]        (Y2)        at (2, -2)                  {$Y_2$};
                    
                    \draw[-latex] (X1.south) -- (Y1.north) ;
                    \draw[-latex] (X2) -- (Y2);
                    \draw[dashed,style={<->}] (X1) -- (X2);
                    \draw[dashed,style={<->}] (Y1) -- (Y2);
                \end{tikzpicture}
                \caption{$\G_{\V}$}
                \label{fig: Gv}
            \end{subfigure}
            
            \begin{subfigure}[b]{.23\textwidth}
                \centering
                \begin{tikzpicture}[
                roundnode/.style={circle, draw=black!60,, fill=white, thick, inner sep=1pt},
                dashednode/.style = {circle, draw=black!60, dashed, fill=white, thick, inner sep=1pt},
                ]
                    \node[roundnode]        (X1)        at (0, 0)                   {$X_1$};
                    \node[roundnode]        (X2)        at (2, 0)                   {$X_2$};
                    \node[dashednode]       (U1)        at (1, 0)                   {$U_1$};
                    
                    \draw[latex-, dashed] (X1) -- (U1);
                    \draw[-latex, dashed] (U1) -- (X2);
                \end{tikzpicture}
                \caption{$\G[X_1,X_2]$}
                \label{fig: G[X1,X2]}
            \end{subfigure}
            \hfill
            \begin{subfigure}[b]{.23\textwidth}
                \centering
                \begin{tikzpicture}[
                roundnode/.style={circle, draw=black!60,, fill=white, thick, inner sep=1pt},
                dashednode/.style = {circle, draw=black!60, dashed, fill=white, thick, inner sep=1pt},
                ]
                    \node[roundnode]        (X1)        at (0, 0)                   {$X_1$};
                    \node[roundnode]        (X2)        at (2, 0)                   {$X_2$};
                    
                    \draw[dashed,style={<->}] (X1) -- (X2);
                \end{tikzpicture}
                \caption{$\G_{\{X_1,X_2\}}$}
                \label{fig: G_X1,X2}
            \end{subfigure}
            \caption{An example for a causal DAG $\G$ over observed variables $\V = \{X_1,X_2,Y_1,Y_2\}$ and unobserved variables $\U=\{U_1,U_2\}$.}
            \label{fig: counterexample}
        \end{figure}
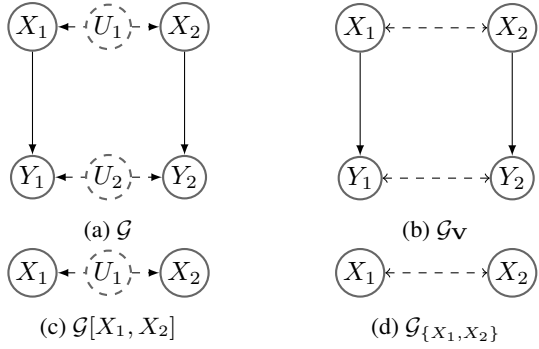
    
        \textbf{Example 1:}
        Consider the causal graph $\G$ in Figure \ref{fig: G}, where $\V = \{X_1,X_2,Y_1,Y_2\}$ and $\U = \{U_1,U_2\}$.
        $\G_{\V}$ is depicted in Figure \ref{fig: Gv}.
        The c-components of $\V$ are $\{X_1,X_2\}$ and $\{Y_1,Y_2\}$.
        Figure \ref{fig: G[X1,X2]} depicts the inducing subgraph of $\G$ over $\{X_1,X_2\}$ and $\G_{\{X_1,X_2\}}$ is depicted in Figure \ref{fig: G_X1,X2}. 
        Herein, $\G[\{X_1,X_2\}]$ is $\{X_1,X_2\}$-rooted c-forest since $\G_{\{X_1,X_2\}}$ is single c-component.
        
    \subsection{Identifiability}    
        The goal in the \emph{identifiability} problem is to understand whether a post-interventional distribution can be uniquely computed from observational distribution $P(\V)$, given the causal graph \citep{pearl2009causality}.
        
        \begin{definition} [\textbf{identifiability}] \label{def: id}
            Suppose $\X$ and $\Y$ are two disjoint subsets of $\V$.
            The causal effect of $\X$ on $\Y$ is said to be identifiable from $\G$ if for any $\x \in \dom{\X}$ and $\y \in \dom{\Y}$, $P^{\M}_{\x}(\y)$ is uniquely computable from $P^{\M}(\mathbf{V})$ in any SEM $\mathcal{M} \in \mathbb{M}^{+}(\G)$.
            Also, $Q[\Y]$ is said to be identifiable from $\G$ if the causal effect of $\V \setminus \Y$ on $\Y$ is identifiable from $\G$.
        \end{definition}
        \cite{huang2008completeness} showed that identifiability of a causal effect is equivalent to identifiability of a specific $Q[\cdot]$.
        \begin{proposition}[\cite{huang2008completeness}]\label{prp: 1}
            The causal effect of $\X$ on $\Y$ is identifiable from $\G$ if and only $Q[\Anc{\Y}{\G_{\V \setminus \X}}]$ is identifiable from $\G$.
        \end{proposition}
        For a subset $\textbf{S}$ of observed nodes, \citep{tian2003ID} showed that identifiability of a $Q[\textbf{S}]$ is equivalent to identifiability of all its c-components.
        \begin{proposition}[\cite{tian2003ID}] \label{prp: 2}
            Suppose $\mathbf{S}_1,\cdots,\mathbf{S}_l$ are the c-components of $\mathbf{S}\subseteq \V$. $Q[\mathbf{S}]$ is identifiable from $\G$ if and only if $Q[\mathbf{S}_i]$ is identifiable from $\G$ for all $i\in [1:l]$.
        \end{proposition}
        
        \begin{algorithm}[t]
            \caption{Identifiability}
            \label{algo: ID}
            \begin{algorithmic}[1]
                \STATE \textbf{Function ID}($\X,\Y,\G$)
                \STATE \textbf{Output:} True, if the causal effect of $\X$ on $\Y$ is identifiable from $\G$. 
                \STATE $\mathbf{S} \gets \Anc{\Y}{\G_{\V \setminus \X}}$
                \STATE $\{\mathbf{S}_1,\dots,\mathbf{S}_l\} \gets$ c-components of $\mathbf{S}$
                \FOR{$i$ from $1$ to $l$}
                    \IF{\textbf{ID\_Single}($\mathbf{S}_i,\G$) = False}
                        \STATE \textbf{Return} False
                    \ENDIF
                \ENDFOR
                \STATE \textbf{Return} True
            \end{algorithmic}
            \hrulefill
            \begin{algorithmic}[1]
                \STATE \textbf{Function ID\_Single}($\mathbf{S},\G$)
                \STATE \textbf{Output:} True, if $Q[\mathbf{S}]$ is identifiable from $\G$, where $\mathbf{S}$ is a single c-component.
                \STATE $\Y \gets \V$
                \WHILE{$\Y \neq \mathbf{S}$}
                    \STATE $\mathbf{A} \gets \Anc{\mathbf{S}}{\G_{\Y}}$
                    \STATE $\Y_{new} \gets $ The c-component of $\mathbf{A}$ that contains $\mathbf{S}$
                    \IF{$\Y_{new} = \Y$}
                        \STATE \textbf{Return} False
                    \ELSE
                        \STATE $\Y \gets \Y_{new}$
                    \ENDIF
                \ENDWHILE
                \STATE \textbf{Return} True
            \end{algorithmic}
        \end{algorithm}
        
        Based on Propositions \ref{prp: 1} and \ref{prp: 2}, \citep{tian2003ID} proposed an algorithm  that for two disjoint subsets $\X$ and $\Y$ checks the identifiability of the causal effect of $\X$ on $\Y$ from observational distribution given the causal graph $\G$.
        As we will use their algorithm as a subroutine in our algorithm for g-identifiability, we present their method in Algorithm \ref{algo: ID}. 
        In this algorithm, function \textbf{ID\_Single} determines whether $Q[\mathbf{S}]$ is identifiable from $\G$ when $\mathbf{S}$ is a single c-component. 
        More precisely, this function starts from $\Y = \V$ and at each step, it decreases $\Y$ such that both $Q[\Y]$ remains identifiable from $\G$ and $\mathbf{S} \subseteq \Y$.
        If this procedure can reduce $\Y$ to $\mathbf{S}$, then $Q[\mathbf{S}]$ is identifiable otherwise,  $Q[\mathbf{S}]$ is not identifiable. 
        This algorithm is both sound and complete \citep{shpitser2006identification, huang2008completeness}.
        
    \subsection{General identifiability}
        In the previous section, we explained the classical identifiability problem which determines whether a causal effect is identifiable from observational distribution given the causal graph. 
        As we discussed earlier, in many problems of interest, the goal is to identify a causal effect from a set of both observational and interventional distributions given a causal graph. 
        A variant of this problem was defined by \cite{lee2019general} under the name g-identifiability.
        
        \begin{definition}[g-identifiability in  \citep{lee2019general}]\label{def: gid_old}
            Let $\X, \Y$ be disjoint subsets of $\V$, $\mathbb{Z}=\{\Z_i\}_{i=0}^{m}$ be a collection of subsets of $\V$, and let $\G$ be a causal diagram. 
            $P_{\x}(\y)$ is said to be g-identifiable from $\mathbb{Z}$ in $\G$, if $P_{\mathbf{x}}(\mathbf{y})$ is uniquely computable from distributions $\{P(\V|do(\z))\}_{\Z\in \mathbb{Z}, \z\in \dom{\Z}}$ in any causal model which induces $\G$.
        \end{definition}
        
        Note that the causal model in this definition belongs to $\mathbb{M}(\G)$.
        However, as we shall discuss in Section \ref{sec: positivity}, it is crucial to assume that the causal model is positive, i.e., it belongs to $\mathbb{M}^+(\G)$.
        Therefore, we modify the above definition as follows.
    
        \begin{definition}[g-identifiability] \label{def: gid}
            Suppose $\mathbb{A}=\{\A_i\}_{i=0}^m$ is a collection of subsets of $\V$ and $\X, \Y$ are two disjoint subsets of $\V$.
            The causal effect of $\X$ on $\Y$ is said to be g-identifiable from $(\mathbb{A}, \G)$ if for any $\x \in \dom{\X}$ and $\y \in \dom{\Y}$, $P^{\M}_{\x}(\y)$ is uniquely computable from the set of distributions $\{Q[\A_i]\}_{i=0}^m$ in any SEM $\M \in \mathbb{M}^{+}(\G)$.
            Also, $Q[\Y]$ is said to be g-identifiable from $(\mathbb{A}, \G)$ if the causal effect of $\V \setminus \Y$ on $\Y$ is g-identifiable from $(\mathbb{A}, \G)$.
        \end{definition}
        
        Note that knowing $P(\V|do(\Z))$ for some subset $\Z\subseteq\V$ is equivalent to knowing $Q[\V \setminus \Z]$, and therefore, by setting $\A_i = \V \setminus \Z_i$, the two aforementioned definitions are the same except for the positivity assumption.
        For the remainder of this paper, we use Definition \ref{def: gid} for g-identifiability.
    
\section{On the positivity assumption in g-identifiability} \label{sec: positivity}
    In our definition of g-identifiability and the classical definition of identifiability (Definitions \ref{def: gid} and \ref{def: id}), only SEMs that belong to $\mathbb{M}^+(\G)$ instead of $\mathbb{M}(\G)$ are considered \citep{huang2008completeness,shpitser2006identification}.
    That is, SEMs with positive probabilities for any realization $\mathbf{v} \in \dom{\V}$. 
    In this section, we discuss why this assumption is crucial by showing that ignoring positivity leads to wrong conclusions.
    As a consequence, since \cite{lee2019general} presented the soundness and completeness of their algorithm for g-identifiability, ignoring the positivity assumption, we discuss how after imposing the assumption, their results are no longer valid.
    We further show that this issue cannot be fixed by the relaxed version of the positivity assumption introduced by \citep{shpitser2006identification}.
    After this discussion, we present a new algorithm in the next section for g-identifiability and prove its soundness and completeness under the positivity assumption.
    
    \subsection{Soundness requires positivity}
        The following example shows that do-calculus-based methods (e.g., Algorithm \ref{algo: ID}) are no longer sound for the ID problem ignoring the positivity assumption.
        
        \textbf{Example 2:}
        Consider again the causal graph in Figure \ref{fig: counterexample}.
        Herein, do-calculus-based methods (e.g., Algorithm \ref{algo: ID}) would report that the causal effect of $\mathbf{X}=\{X_1,X_2\}$ on $\mathbf{Y}=\{Y_1,Y_2\}$ is identifiable given $\G$.
        However, by ignoring the positivity assumption, we can introduce two SEMs $\M_1$ and $\M_2$ in $\mathbb{M}(\G)$ that have the same observational distribution but result in two different post-interventional distributions after intervening on $\{X_1,X_2\}$. 
        This clearly contradicts with the identifiability of $P_{x_1,x_2}(y_1,y_2)$.
        
        All variables in both models are binary. 
        Also, for both models and $i\in \{1,2\}$, we define $P(U_i=0) = P(U_i=1) = 0.5$ and $X_i = U_1$.
        In model $\M_1$, we define $Y_1,Y_2$ to have the following conditional distributions:
        \begin{equation*}
            \begin{split}
                P^{\M_1}(y_1 \mid u_2, x_1) = &\frac{1}{3} \mathds{1}_{y_1=u_2} + \frac{2}{3}\mathds{1}_{y_1\neq u_2},
                \\
                P^{\M_1}(y_2 \mid u_2, x_2) = &\frac{1}{3} \mathds{1}_{y_2 = (u_2\oplus x_2)} + \frac{2}{3}\mathds{1}_{y_2\neq(u_2\oplus x_2)},
            \end{split}
        \end{equation*}
        where $\mathds{1}_{A}$ is the indicator function which is one whenever the statement in $A$ is true and is zero otherwise. 
        For model $\M_2$, we define the conditional distributions of $Y_1,Y_2$ as
        \begin{align*}
                &P^{\M_2}(y_1 \mid u_2, x_1) = \frac{2}{3}\mathds{1}_{y_1=(u_2\oplus x_1)} + \frac{1}{3}\mathds{1}_{y_1\neq(u_2\oplus x_1)},\\
                &P^{\M_2}(y_2|u_2, x_2) = \frac{2}{3}\mathds{1}_{y_2=u_2} + \frac{1}{3}\mathds{1}_{y_2\neq u_2}.
        \end{align*}
        
        It is straightforward to see that for any realizations $(x_1,x_2,y_1,y_2) \in \dom{\V}$, we have
        \begin{equation*}
            P^{\M_1}(x_1,x_2,y_1,y_2) = P^{\M_2}(x_1,x_2,y_1,y_2).
        \end{equation*}
        However,
        \begin{equation*}
        \begin{split}
            \frac{4}{9}=P^{\M_1}_{x_1=0, x_2=1}(Y_1=0, Y_2=0) \\
            \neq 
            P^{\M_2}_{x_1=0, x_2=1}(Y_1=0, Y_2=0) = \frac{5}{9}.
        \end{split}
        \end{equation*}
        Note that $\M_1$ and $\M_2$ do not belong to $\mathbb{M}^+(\G)$, since $P(x_1=0, x_2=1, y_1, y_2)= 0$ for any $y_1\in \dom{Y_1}$ and $y_2\in \dom{Y_2}$.
        This example shows that if we use $\mathbb{M}(\G)$ instead of $\mathbb{M}^+(\G)$ in Definition \ref{def: id}, the causal effect of $\mathbf{X}$ on $\mathbf{Y}$ is not identifiable from $\G$, and therefore, do-calculus-based methods such as the proposed algorithm in \cite{lee2019general} are not sound.
        Specifically, the proposed algorithm in \cite{lee2019general} suggests the causal effect in this example is g-identifiable and returns the following expression:
        \begin{equation*}
            P_{x_1, x_2}(y_1,, y_2) = P(y_1|x_1, x_2)P(y_2|y_1, x_2, x_1).
        \end{equation*}
        This expression is not well-defined for all realizations ignoring the positivity assumption because for some realizations $P(x_1, x_2)$ is zero which means the conditional distribution $P(y_1|x_1, x_2)$ is not well-defined.
        Thus, the algorithm in \cite{lee2019general} \textit{is not sound}.

        Next, we discuss the g-identifiability in \cite{lee2019general} and show that the completeness result provided in that work relies on two models in $\mathbb{M}(\G)$ that violate the positivity assumption.
    
    \subsection{Completeness}
        \citep{lee2019general} presented necessary and sufficient conditions to determine if a causal effect $P_{\textbf{x}}(\textbf{y})$ is g-identifiable w.r.t. the Definition \ref{def: gid_old}.
        To prove that their proposed conditions are necessary for g-identifiability, they construct two models $\M_1$ and $\M_2$ such that the available distributions in the definition of the problem are the same for both models yet $P^{\M_1}_{\textbf{x}}(\textbf{y})\neq P^{\M_2}_{\textbf{x}}(\textbf{y})$.
        The issue here is that they constructed their models ignoring the positivity assumption, allowing for zero probability for some realizations. 
        In fact, having zero probabilities in their model is essential for the proof. 
        For instance, Lemma 3 in \cite{lee2019general} states that under certain conditions, there is an observed variable $R\in \V$ such that it takes value zero in both their models with probability one.
        In other words, the probability of $R$ not being zero is zero (see Appendix \ref{apn: pos assumption simple} for more details.)
        This shows that adding the positivity assumption to the definition of gID will fail the proof technique in \citep{lee2019general} for the completeness of their proposed algorithm.
        
        It is noteworthy to mention that an alternative positivity assumption is introduced by \cite{shpitser2006identification}. 
        Below, we describe this assumption and discuss that the models introduced in \cite{lee2019general} also violate this assumption.
    
    \subsection{Relaxed positivity assumption}
        \cite{shpitser2006identification} show that in the ID problem of a causal effect \mbox{$P_{\mathbf{x}}(\mathbf{y})$}, one can relax the positivity constraint \mbox{$P(\mathbf{V})>0$} to \mbox{$P(\mathbf{X}|\left(\Pa{\mathbf{X}}{\G}\cap\mathbf{V}\right) \setminus \mathbf{X})>0$}.
        They show that the rules of do-calculus are sound under the relaxed positivity assumption.
        However, as we mentioned, even the relaxed constraint does not hold for the constructed models in \mbox{\cite{lee2019general}}.
        More precisely, consider the causal graph $\G$ in \mbox{Figure \ref{fig:countr_exmpl_complex}} which is brought here from \cite{lee2019general}. 
        Assume that we are interested in g-identifying the causal effect $Q[R]$ from $\mathbb{Z}=\{\varnothing\}$, i.e., from mere observational distribution $P(\V)$, w.r.t. Definition \ref{def: gid_old}.
        In this case, $\X=\{T_1, T_2, T_3\}$ and therefore:
        \begin{equation*}
            P(\mathbf{X}|\left(\Pa{\mathbf{X}}{\G}\cap\mathbf{V}\right) \setminus \mathbf{X}) = P(T_1, T_2, T_3).  
        \end{equation*}
        
        The result in \cite{lee2019general} implies that the causal effect $Q[R]$ is not g-identifiable given the causal graph $\G$ in Figure \ref{fig:countr_exmpl_complex}. 
        To prove the non g-identifiability, \cite{lee2019general} constructed two models $\M_1$ and $\M_2$ that impose similar observational distributions, i.e., $P^{\M_1}(\V)=P^{\M_2}(\V)$, while the causal effect $Q[R]$ under these two models are not the same for at least one realization. 
        Next, we present these two models and show that they violate the positivity assumption claimed in \cite{shpitser2006identification}, i.e., $P(T_1,T_2,T_3)$ is zero for certain realizations of $\{T_1,T_2,T_3\}$.
        
        By the construction in \cite{lee2019general}, variables $T_3, U_1, U_2, U_3$ are binary variables and $T_1, T_2$ are binary vectors of length two. 
        For both models, all unobserved variables are defined to be binary with uniform distribution, and the observed variables $T_1, T_2, T_3$ are defined as follows.
        \begin{align*}
            &T_3 = U_2 \oplus U_3,\\
            &T_{2, 1} = T_3,\quad  T_{2, 2} = U_1,\\
            &T_{1, 1} = T_{2, 1}\oplus U_2,\quad T_{1, 2} = T_{2, 2}.  
        \end{align*}
        In model $\M_1$, variable $R$ is defined as
        \begin{equation*}
            R = \mathds{1}_{T_{1, 1} = 0} \wedge \mathds{1}_{T_{1, 2}=0} \wedge \mathds{1}_{U_3=1} \wedge \mathds{1}_{U_{1} = 1},
        \end{equation*}
        and in model $\M_2$, it is defined to be zero, i.e., $R = 0.$
        
        Given the above models, it is clear that the probability $P(t_1, t_2, t_3)$ is equal to zero whenever $t_{2,1}\neq t_3$, and therefore, the relaxed positivity constraint \mbox{$P(T_1, T_2, T_3)>0$} does not hold for the models in \cite{lee2019general}.
        See Appendix \ref{sec: apd_pos} for more details.
        
        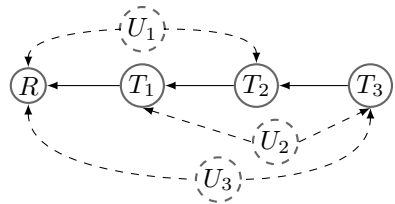
\begin{figure}
            \centering
            \begin{tikzpicture}[
            roundnode/.style={circle, draw=black!60,, fill=white, thick, inner sep=1pt},
            dashednode/.style = {circle, draw=black!60, dashed, fill=white, thick, inner sep=1pt},
            ]
            \node[roundnode]        (t1)        at (0, 0)                   {$T_1$};
            \node[roundnode]        (t2)        at (1.5, 0)                 {$T_2$};
            \node[roundnode]        (t3)        at (3, 0)                   {$T_3$};
            \node[roundnode]        (r)         at (-1.5, 0)                {$R$};
            \node[dashednode]       (u1)        at (0, 0.75)                {$U_1$};
            \node[dashednode]       (u2)        at (1.75, -0.75)            {$U_2$};
            \node[dashednode]       (u3)        at (1, -1.25)               {$U_3$};
            
            \draw[-latex] (t1.west) -- (r.east) ;
            \draw[-latex] (t2.west) -- (t1.east);
            \draw[-latex] (t3.west) -- (t2.east);
            \draw[latex-, dashed] (t1.south) -- (u2.west);
            \draw[-latex, dashed] (u2.east) -- (t3.south);
            \draw[latex-, dashed] (r.north) .. controls +(up:5mm) and +(right:2mm) .. (u1.west);
            \draw[latex-, dashed] (t2.north) .. controls +(up:5mm) and +(left:2mm) .. (u1.east);
            \draw[latex-, dashed] (r.south) .. controls +(down:10mm) and +(right:2mm) .. (u3.west);
            \draw[latex-, dashed] (t3.south) .. controls +(down:10mm) and +(left:2mm) .. (u3.east);
                
            \end{tikzpicture}
        \caption{A causal graph of \citep{lee2019general} that shows the violation of relaxed positivity assumption in constructed models of \citep{lee2019general}.}
        \label{fig:countr_exmpl_complex}
        \end{figure}
        
        To summarize, in this section, our goal was to prove the importance of positivity assumption in both classical ID and its generalization gID.
        We did so by showing that the rules of \textit{do-calculus} and consequently the proposed algorithm in \cite{lee2019general} are not sound without the positivity assumption. Moreover, we discussed that the completeness proof in \cite{lee2019general} only holds when there is no positivity assumption.
        This motivates our work to revisit the gID problem by including the positivity assumption in the definition of gID and presenting a new algorithm that is provably sound and complete.

\section{An algorithm for GID} \label{sec: algorithm}
    In this section, we propose an algorithm for gID from $(\mathbb{A}, \G)$, where $\mathbb{A}=\{\A_i\}_{i=0}^m$ is a collection of subsets of $\V$.
    To this end, we first extend Propositions \ref{prp: 1} and \ref{prp: 2} from identifiability to g-identifiability.
    
    \begin{proposition} \label{prp: 3}
        Let $\X$ and $\Y$ be two disjoint subsets of $\V$.
        The causal effect of $\X$ on $\Y$ is g-identifiable from $(\mathbb{A},\G)$ if and only if $Q[\Anc{\Y}{\G_{\V \setminus \X}}]$ is g-identifiable from $(\mathbb{A},\G)$.
    \end{proposition}
    \begin{proposition} \label{prp: 4}
        Suppose $\mathbf{S}_1,\cdots,\mathbf{S}_l$ are the c-components of $\mathbf{S} \subseteq \V$. 
        $Q[\mathbf{S}]$ is g-identifiable from $(\mathbb{A}, \G)$ if and only if $Q[\mathbf{S}_i]$ is g-identifiable from $(\mathbb{A}, \G)$ for all $i\in [1:l]$.
    \end{proposition}
    Proofs are provided in Appendix \ref{sec: apd_proof}. 
    Proposition \ref{prp: 3} allows us to solve the gID problem for $P_{\x}(\Y)$ by solving an equivalent problem for $Q[\mathbf{S}]$, where $\mathbf{S}$ is given in the same proposition. 
    Proposition \ref{prp: 4} shows that the g-identifiability of $Q[\mathbf{S}]$ from $(\mathbb{A}, \G)$ is equivalent to g-identifiability of its single c-components. 
    The following result provides a method for solving g-identifiability of $Q[\mathbf{S}]$ when $\textbf{S}$ is a single c-component.
    \begin{theorem} \label{thm: main}
        Suppose $\mathbf{S} \subseteq \V$ is a single c-component. 
        $Q[\mathbf{S}]$ is g-identifiable from $(\mathbb{A}, \G)$ if and only if there exists $\A \in \mathbb{A}$ such that $\mathbf{S} \subseteq \A$ and $Q[\mathbf{S}]$ is identifiable from $\G[\A]$.
    \end{theorem}
    \begin{algorithm}[t]
        \caption{g-identifiability}
        \label{algo: GID}
        \begin{algorithmic}[1]
            \STATE \textbf{Function GID}($\X,\Y,\mathbb{A}=\{\A_i\}_{i=0}^m ,\G$)
            \STATE \textbf{Output:} True, if the causal effect of $\X$ on $\Y$ is g-identifiable from $(\mathbb{A},\G)$. 
            \STATE $\mathbf{S} \gets \Anc{\Y}{\G_{\V \setminus \X}}$
            \STATE $\{\mathbf{S}_1,\dots,\mathbf{S}_l\} \gets$ c-components of $\mathbf{S}$
            \FOR{$i$ from $1$ to $l$}
                \IF{\textbf{GID\_Single}($\mathbf{S}_i,\mathbb{A}=\{\A_i\}_{i=0}^m ,\G$) = False}
                    \STATE \textbf{Return} False
                \ENDIF
            \ENDFOR
            \STATE \textbf{Return} True
        \end{algorithmic}
        \hrulefill
        \begin{algorithmic}[1]
            \STATE \textbf{Function GID\_Single}($\mathbf{S}, \mathbb{A}=\{\A_i\}_{i=0}^m, \G$)
            \STATE \textbf{Output:} True, if $Q[\mathbf{S}]$ is identifiable from $(\mathbb{A},\G)$, where $\mathbf{S}$ is a single c-component.
            \FOR{$i$ from $0$ to $m$}
                \IF{$\mathbf{S} \subseteq \A_i$ and \textbf{ID\_Single}($\mathbf{S}, \G[\A_i]$) = True}
                    \STATE \textbf{Return} True
                \ENDIF
            \ENDFOR
            \STATE \textbf{Return} False
        \end{algorithmic}
    \end{algorithm}

    A proof for Theorem \ref{thm: main} is provided in Section \ref{sec: main}.
    Note that the equivalent condition provided in Theorem \ref{thm: main} is identifiability of a $Q[\cdot]$. This can be checked by function \textbf{ID\_Single} in Algorithm \ref{algo: ID}. 
    Therefore, when $\textbf{S}$ is a single c-component, in order to check whether $Q[\mathbf{S}]$ is g-identifiable from $(\mathbb{A}, \G)$, we need to check the identifiability of $Q[\mathbf{S}]$ from $\G[\A]$ for all $\A\in \mathbb{A}$ that $\mathbf{S}\subseteq \A$. 
    Algorithm \ref{algo: GID} summarizes the steps for solving g-identifiability of a causal effect given $(\mathbb{A}, \G)$.
    \begin{theorem}
        Algorithm \ref{algo: GID} is sound and complete. 
    \end{theorem}
    \begin{proof}
        It directly follows from Propositions \ref{prp: 3} and \ref{prp: 4} and Theorem \ref{thm: main}.
    \end{proof}
    \begin{remark}
        Under the relaxed positivity assumption, the algorithm is still sound and complete because Algorithm \ref{algo: GID} is based on the rules of do-calculus, and these rules are both sound and complete under the relaxed positivity assumption.
    \end{remark}
    Suppose Algorithm \ref{algo: GID} determines that the causal effect of $\X$ on $\Y$ is g-identifiable from $(\mathbb{A}, \G)$. 
    Analogous to the method in \cite{tian2003ID}, we can derive a formula for $P_{\x}(\Y)$ as follows.
    For each $\mathbf{S}_i \in \{\mathbf{S}_1,\cdots,\mathbf{S}_l\}$, we can derive a formula for $Q[\mathbf{S}_i]$ using \textbf{ID\_Single} function in line $4$ of \textbf{GID\_Single}. 
    This allows us to compute $Q[\mathbf{S}]$ using 
    \begin{equation*}
        Q[\mathbf{S}] = \prod_{i=1}^l Q[\mathbf{S}_i].
    \end{equation*}
    Finally, the expression for $P_{\x}(\Y)$ will be
    \begin{equation*}
        P_{\x}(\Y) = \sum_{\mathbf{S}\setminus \Y} Q[\mathbf{S}].
    \end{equation*}

\section{Main result: Theorem 1} \label{sec: main}
    In this section, we present the main steps of the proof of Theorem \ref{thm: main}.
    The technical lemmas in this section are proved in Appendix \ref{sec: apd_proof}. 

    \paragraph{Sufficient part:}
    This part is straightforward: if $Q[\mathbf{S}]$ is identifiable from $\G[\A]$ for some $\A \in \mathbb{A}$ such that $\mathbf{S} \subseteq \A$, then $Q[\mathbf{S}]$ is uniquely computable from $Q[\A]$, and therefore, $Q[\mathbf{S}]$ is g-identifiable from $(\mathbb{A},\G)$.
    
    \paragraph{Necessary part:}
    Suppose $\mathbf{S}$ is a single c-component and  $Q[\mathbf{S}]$ is not identifiable from $\G[\A_i]$ for all $\A_i \in \mathbb{A}$ such that $\mathbf{S} \subseteq \A_i$. 
    We need to show that $Q[\mathbf{S}]$ is not g-identifiable from $(\mathbb{A},\G)$.
    Recall that $\mathbb{A}=\{\A_i\}_{i=0}^m$.
    To this end, we will introduce two SEMs $\M_1$ and $\M_2$ in $\mathbb{M^+(\G)}$ such that for each $i\in [0:m]$ and any $\mathbf{v} \in \dom{\V}$,
    \begin{equation} \label{eq: same Q over A}
        Q^{\M_1}[\mathbf{A}_i](\mathbf{v}) = Q^{\M_2}[\mathbf{A}_i](\mathbf{v}), 
    \end{equation}
    but there exists $\mathbf{v}_0 \in \dom{\V}$ such that 
    \begin{equation} \label{eq: diff Q over S}
        Q^{\M_1}[\mathbf{S}](\mathbf{v}_0) \neq Q^{\M_2}[\mathbf{S}](\mathbf{v}_0).
    \end{equation}
    This shows that $Q[\mathbf{S}]$ cannot be uniquely computed from $\{Q[\A_i]\}_{i=0}^m$.
    
    For sake of space, we assume that there exists at least one $i\in [0, m]$ such that $\mathbf{S}\subset \mathbf{A}_i$. 
    In this case, without loss of generality, we assume that there exists $k \in [0:m]$ such that  $\mathbf{S}\subset \A_i$ for $i\in [0:k]$ and $\mathbf{S}\nsubseteq \A_i$ for $i\in [k+1:m]$.
    A proof for the case in which $\mathbf{S}$ is not a subset of $\A_i$ for all $i\in [0, m]$ is provided in Appendix \ref{sec: apd_second case}.
    
    We first modify $\G$ by deleting some nodes and edges and show that it is enough to prove Theorem \ref{thm: main} for the modified graph. 
    Then, we provide our method for constructing $\M_2$ from $\M_1$ by introducing a system of linear equations.

    \paragraph{Graph modification:}\label{sec: graph-mod}
    Since $\mathbf{S}$ is single c-component, the bidirected edges in $\G_{\mathbf{S}}$ form a connected graph over $\mathbf{S}$. 
    Let $\F^{\mathbf{S}}$ be a minimal (in terms of edges) spanning subgraph of $\G[\mathbf{S}]$ such that $\F^{\mathbf{S}}_{\mathbf{S}}$ is single c-component. 
    Thus, $\F^{\mathbf{S}}_{\mathbf{S}}$ has no directed edges, and its bidirected edges form a spanning tree.
    \begin{lemma}[\cite{shpitser2006identification}] \label{lem: c-forest existence}
        Suppose \mbox{$\mathbf{S} \subseteq\A \subseteq \V$}. 
        $Q[\mathbf{S}]$ is not identifiable from $\G[\A]$ if and only if there exists at least one $\mathbf{S}$-rooted c-forest  $\mathcal{F}$ with the set of observed variables $\B$ such that $\mathbf{S}\subsetneq \B \subseteq \A$, the bidirected edges of $\F_{\B}$ form a spanning tree, and the induced subgraph of $\mathcal{F}$ over $\mathbf{S}$ is $\mathcal{F}^{\mathbf{S}}$, i.e.,  $\mathcal{F}^{\mathbf{S}}= \mathcal{F}[\mathbf{S}]$.
    \end{lemma}
    Recall that for each $i \in [0:k]$, $\mathbf{S}\subset \A_i$ and $Q[\mathbf{S}]$ is not identifiable from $Q[\A_i]$. 
    Hence, Lemma \ref{lem: c-forest existence} implies that for each $i\in [0:k]$, there exists a $\mathbf{S}$-rooted c-forest $\mathcal{F}_i$ over a set of observed variables $\B_i$ such that $\mathbf{S}\subsetneq \B_i \subseteq \A_i$, the bidirected edges of $(\mathcal{F}_i)_{\mathbf{S}}$ form a spanning tree, and $\mathcal{F}^{\mathbf{S}} = \mathcal{F}_i[\mathbf{S}]$. 
    Next, we use $\{\mathcal{F}_i\}_{i=0}^k$ to modify $\G$.

    We define $\G'$ to be the union of all the subgraphs in $\{\mathcal{F}_i\}_{i=0}^k$ with the observed variables $\V' := \bigcup_{i=0}^k \B_i$ and unobserved variables $\U'$.
    Furthermore, let $\mathbb{A}':= \{\A_i':=\A_i \cap \V'\}_{i=0}^m$. 
    Because for each $i\in [0:k]$, $\mathcal{F}_i$ is a $\mathbf{S}$-rooted c-forest in $\G'$, Lemma \ref{lem: c-forest existence} implies that $Q[\mathbf{S}]$ is not identifiable from $\G'[\A_i']$.

    Next result establishes the connection between non g-identifiability of $Q[\mathbf{S}]$ from $(\mathbb{A}, \G)$ and non g-identifiability of $Q[\mathbf{S}]$ from $(\mathbb{A}', \G')$.
    \begin{lemma} \label{lem: simplify}
        If $Q[\mathbf{S}]$ is not g-identifiable from $(\mathbb{A}', \G')$, then $Q[\mathbf{S}]$ is not g-identifiable from $(\mathbb{A}, \G)$.
    \end{lemma}
    To complete the proof using Lemma \ref{lem: simplify}, it is enough to show that $Q[\mathbf{S}]$ is not g-identifiable from $(\mathbb{A}', \G')$.
    
    \paragraph{From g-identifiability to a system of linear equations:}\label{sec: linearequations}
    To show  that $Q[\mathbf{S}]$ is not g-identifiable from $(\mathbb{A}', \G')$, we introduce two models in $\mathbb{M^+(\G')}$ such that equations \eqref{eq: same Q over A} and \eqref{eq: diff Q over S} are satisfied. 
    That is, $Q[\mathbf{S}]$ cannot be uniquely computed from $\{Q[\A'_i]\}_{i=0}^m$.
    
    Note that to define a SEM $\M$ over a causal graph $\G'$, it suffices to define the domains $\dom{X}$ and either the conditional distributions $P^{\M}(X|\Pa{X}{\G'})$ or the corresponding equation in the SEM for all  $X\in \V'\cup \U'$, where $\V'$ and $\U'$ denote the observed and unobserved variables in $\G'$. 
    We define the domains of all variables to be finite, i.e., $|\dom{X}|<\infty$ for all $X\in\V'\cup \U'$.
    Let $U_0 \in \U'$ be a fixed unobserved variable (we will discuss later how to select $U_0$) with domain $\dom{U_0} := \{\gamma_1,\cdots,\gamma_d\}$.
    We define both models $\M_1$ and $\M_2$ to have similar distributions over all variables except variable $U_0$ (We will specify these distributions in Section \ref{sec: M1}.)
    More specifically, for all $V \in \V'$,
    \begin{equation} \label{eq: same dist 1}
        P^{\M_1}(V \mid \Pa{V}{\G'}) = P^{\M_2}(V \mid \Pa{V}{\G'}),
    \end{equation}
    and for all $U \in \U' \setminus \{U_0\}$,
    \begin{equation} \label{eq: same dist 2}
        P^{\M_1}(U) = P^{\M_2}(U)=\frac{1}{|\dom{U}|}.
    \end{equation}
    As the distributions in Equations \eqref{eq: same dist 1} and \eqref{eq: same dist 2} are the same for both models, for the sake of brevity, we drop the superscripts $\M_1$ and $\M_2$ from here on.
    For $j\in[1:d]$, We define $P^{\M_1}(U_0=\gamma_j)=1/d$ and $P^{\M_2}(U_0 = \gamma_j)=p_j$, where we will specify $\{p_j\}_{j=1}^d$ later such that $\M_2\in\mathbb{M}^+(\G')$ and both Equations \eqref{eq: same Q over A} and \eqref{eq: diff Q over S} hold.
    
    For $\mathbf{v} \in \dom{\V'}$, $i\in [0:m]$, and $j\in [1:d]$, we define
    \begin{equation*}
        \theta_{i, j}(\mathbf{v}) :=\! \sum_{\U' \setminus \{U_0\}} \prod_{X \in \A_i'} P(x \mid \Pa{X}{\G'})\! \prod_{U\in \U' \setminus \{U_0\}}\! P(u),
    \end{equation*}
    \begin{equation*}
        \eta_j(\mathbf{v}) :=\! \sum_{\U'\setminus \{U_0\}} \prod_{X \in \mathbf{S}} P(x \mid \Pa{X}{\G'})\! \prod_{U\in \U'\setminus \{U_0\}}\! P(u),
    \end{equation*}
    where the index $j$ indicates that $U_0=\gamma_j$ in the factorizations. 
    Using these definitions, we can write $\{Q[\A'_i]\}_{i=0}^m$ and $Q[\textbf{S}]$ for both models $\M_1$ and $\M_2$ as follows:
    \begin{equation} \label{eq: Q[A] with c}
    \begin{split}
        Q^{\M_1}[\mathbf{A}'_i](\mathbf{v}) 
        = \sum_{j=1}^d \frac{1}{d} \theta_{i,j}(\mathbf{v}), \\
        Q^{\M_2}[\mathbf{A}'_i](\mathbf{v}) 
        = \sum_{j=1}^d p_j \theta_{i,j}(\mathbf{v}),
    \end{split}
    \end{equation}
    and
    \begin{equation} \label{eq: Q[S] with c}
    \begin{split}
        Q^{\M_1}[\mathbf{S}](\mathbf{v}) 
        &= \sum_{j=1}^d \frac{1}{d} \eta_j(\mathbf{v}), \\
        Q^{\M_2}[\mathbf{S}](\mathbf{v}) 
        &= \sum_{j=1}^d p_j \eta_j(\mathbf{v}).
    \end{split}
    \end{equation}
    As we mentioned, we need to define $\{p_j\}_{j=1}^d$ such that $\M_2 \in \mathbb{M^+(\G')}$ and both Equations \eqref{eq: same Q over A} and \eqref{eq: diff Q over S} hold.
    Substituting Equations \eqref{eq: Q[A] with c} and \eqref{eq: Q[S] with c} into \eqref{eq: same Q over A} and \eqref{eq: diff Q over S} yield the following set of equations.
    \begin{equation} \label{eq: linear system 1}
    \begin{split}
        &\sum_{j=1}^d (p_j - \frac{1}{d}) \theta_{i,j}(\mathbf{v}) =0, \hspace{0.2cm}\forall \mathbf{v} \in \dom{\V'}, i\in [0,m],\\
        &\sum_{j=1}^d (p_j - \frac{1}{d}) \eta_j(\mathbf{v}_0) \neq 0, \hspace{0.2cm} \exists \mathbf{v}_0 \in \dom{\V;},\\
        & \sum_{j=1}^d p_j = 1,\\
        & 0<p_j<1, \hspace{0.2cm} \forall j \in [1:d].
    \end{split}
    \end{equation}
    Note that the last inequalities ensure that $\M_2 \in \mathbb{M^+(\G')}$.
    The system of linear equations in \eqref{eq: linear system 1} is solvable with respect to  $\{p_j\}_{j=1}^d$ if and only if the following system of linear equations is solvable with respect to $\{\beta_j\}_{j=1}^d$.
    \begin{equation} \label{eq: linear system 2}
    \begin{split}
        &\sum_{j=1}^d \beta_j \theta_{i,j}(\mathbf{v}) =0, \hspace{0.2cm}\forall \mathbf{v} \in \dom{\V'}, i \in [0:m]\\
        &\sum_{j=1}^d \beta_j \eta_j(\mathbf{v}_0) \neq 0, \hspace{0.2cm} \exists \mathbf{v}_0 \in \dom{\V'}\\
        & \sum_{j=1}^d \beta_j=0.
    \end{split}
    \end{equation}
    
    \begin{remark}
        If $\{\beta_j^*\}_{j=1}^d$ is a solution for \eqref{eq: linear system 2}, then
        \begin{equation} \label{eq: system to system}
            p_j^* := \frac{1}{d} + \frac{\beta^*_j}{2h d},
        \end{equation}
        is a solution for \eqref{eq: linear system 1}, where $h = \underset{j \in [1:d]}{max} |\beta^*_j|$. 
        Note that the division by $2h$ in Equation \eqref{eq: system to system} ensures that $0< p_j^* <1$ for each $j \in [1:d]$.
    \end{remark}
    A solution to the system of linear equations in \eqref{eq: linear system 2} will specify the distribution of $U_0$ in model $\M_2$. Clearly, existence of a solution to \eqref{eq: linear system 2} depends on the choices of $\{\theta_{i,j}(\mathbf{v})\}$ and $\{\eta_j(\mathbf{v})\}$. 
    The following result presents a sufficient condition under which \eqref{eq: linear system 2} admits a solution.
    
    For $\mathbf{v}\in \dom{\V'}$ and $i\in [0:m]$, let $\theta_i(\mathbf{v})$ and $\eta(\mathbf{v})$ denote the vectors $(\theta_{i,1}(\mathbf{v}),...,\theta_{i,d}(\mathbf{v}))$  and $(\eta_1(\mathbf{v}),...,\eta_d(\mathbf{v}))$ in $\mathbb{R}^d$, respectively.
    \begin{lemma}\label{lem: lin indep}
        Consider the following set of vectors in $\mathbb{R}^d$
         \begin{equation} \label{eq: lin vectors}
            \mathbf{\Omega}:=\{\theta_{i}(\mathbf{v}):\ i\in [0:m], \mathbf{v}\in \dom{\V'}\} \cup \{\mathds{1}_d\},
        \end{equation}
        where $\mathds{1}_d$ denotes the all-ones vector in $\mathbb{R}^d$. 
        If there exists $\mathbf{v}_0 \in \dom{\V'}$ such that $\eta(\mathbf{v}_0)$ is linearly independent from all the vectors in $\mathbf{\Omega}$, then the system of linear equations in \eqref{eq: linear system 2} admits a solution. 
    \end{lemma}
    To summarize, so far, we have introduced two models for proving the necessary part of Theorem \ref{thm: main}. 
    In order to complete the proof, it remains to specify the conditional distributions in \eqref{eq: same dist 1} for all observed variables which consequently specify the vectors in $\mathbf{\Omega}$ in Equation \eqref{eq: lin vectors} and to find a realization $\mathbf{v}_0 \in \dom{\V'}$ such that $\eta(\mathbf{v}_0)$ is linearly independent from the set of the vectors in $\mathbf{\Omega}$.

    \paragraph{Constructing the conditional distributions:} \label{sec: M1}
    In order to specify the conditional distributions in \eqref{eq: same dist 1}, we first introduce the following definitions and notations. 
    
    Since $\B_0$ is a single c-component, the bidirected edges in $\G_{\B_0}'$ form a connected graph.
    Hence, there exists a bidirected edge between $\mathbf{S}$ and $\B_0 \setminus \mathbf{S}$.
    Accordingly, let $U_0$ be an unobserved variable in subgraph $\mathcal{F}_0$ that has one child in $\mathbf{S}$ and one child in $\T := \V' \setminus \mathbf{S}$.
    We denote the set of unobserved variables in $\G[\mathbf{S}]$ by $\U^{\mathbf{S}}$ and define  $\U^{\T} := \U' \setminus ( \U^{\mathbf{S}} \cup \{U_0\})$. 
    For $X \in \V' \cup \U'$, we define $\alpha(X)$ to denote the number of graphs in $\{\F_j\}_{j=0}^{k}$ that contains $X$.
    
    For each $i\in [0:k]$, let $T_i$ denotes a node in $\B_i \setminus \mathbf{S}$ such that $\Ch{T_i}{\mathcal{F}_i} \cap \mathbf{S} \neq \varnothing$. 
    Note that such variables exist because $\mathcal{F}_i$s are $\mathbf{S}$-rooted c-forest. 

    Now, we are ready to introduce the domains of all variables in $\V'\cup\U'$. Recall that $\V'=\textbf{S}\cup\T$ and $\U'=\U^{\mathbf{S}} \cup \U^{\T} \cup \{U_0\}$.
    \begin{align*}
        &\dom{X} := [0:\kappa], \quad \forall X \in \mathbf{S},\\
        &\dom{X} := \{0, 1\}^{\alpha(T)}, \quad \forall X \in \mathbf{T},\\
        &\dom{X} := [0:\kappa], \quad \forall X \in \U^{\mathbf{S}},\\
        &\dom{X} := \{0, 1\}^{\alpha(U)}, \quad \forall X \in \mathbf{U}^{\T},\\
        &\dom{U_0} := [0:\kappa]\times\{0, 1\}^{\alpha(U_0)-1}.
    \end{align*}
    In the above definition, $\kappa$ is an arbitrary odd integer greater than $4$. 
    Note that the number of elements in $\dom{U_0}$ is $d=(\kappa+1)2^{\alpha(U_0)-1}$.

    According to the above definitions, for each $X\in \T\cup\U^{\T}\cup \{U_0\}$, its domain $\dom{X}$ is a subset of  $\mathbb{R}^{\alpha(X)}$ and it belongs to exactly $\alpha(X)$ number of subgraphs in $\{\F_i\}_{i=0}^k$.
    Suppose $X$ belongs to $\F_{i_1},\cdots, \F_{i_{\alpha(X)}}$, where $i_1 < \cdots < i_{\alpha(X)}$.
    Thus, we denote $X$ by a vector $ (X[i_1],\cdots,X[i_{\alpha(X)}])$ of length $\alpha(X)$. Next, we construct the conditional distributions of the observed variables by specifying their functional dependencies to their parents.
        
    When $X\in \T$, we define the entries of its corresponding vector as
    \begin{equation*}
        X[i_j] \equiv \left(\sum_{Y \in \Pa{X}{\F_{i_j}}} Y[i_j]\right) \pmod{2},
    \end{equation*}
    where $j\in[1:\alpha(X)]$.
    
    We now construct the variables in $\mathbf{S}$. 
    Recall that $U_0$ has one child in $\mathbf{S}$ which we denote it by $S_0$.
    For each $S \in \mathbf{S}\setminus \{S_0\}$ and any realization of $\Pa{S}{\G'}$, we define $\mathds{I}(S)$ to be one if there exists $i\in [0:k]$ such that
    \begin{enumerate}
        \item $T_i \in \Pa{S}{\G'}$ and $T_i[i]=0$, or
        \item there exists $X \in \Pa{S}{\G'} \setminus (\U^{\mathbf{S}}\cup \{T_i\})$ such that $\F_i$ contains $X$ and $X[i]=1$,
    \end{enumerate}
    and zero, otherwise.
    Note that according to the definition of $T_i$, it belongs to $\mathcal{F}_i$ and therefore, $T_i[i]$ exists. 
    Analogously, we define $\mathds{I}(S_0)$ to be one if there exists $i\in [0:k]$ such that
    \begin{enumerate}
        \item $T_i \in \Pa{S}{\G'}$ and $T_i[i]=0$, or
        \item $i\neq0$, $\mathcal{F}_i$ contains $U_0$, and $U_0[i]=1$, or
        \item there exists $X \in \Pa{S}{\G'} \setminus (\U^{\mathbf{S}}\cup \{T_i,U_0\})$ such that $\F_i$ contains $X$ and $X[i]=1$.
    \end{enumerate}
    
    Now, for each $S\in \mathbf{S}$ and $s \in [0:\kappa]$, we define $P(S=s \mid \Pa{S}{\G'})$ as
    \begin{equation*}
    \begin{cases} 
        \frac{1}{\kappa+1} & \text{ if } \mathds{I}(S)=1\\
        1-\kappa\epsilon &  \text{ if } \mathds{I}(S)=0  \text{ and } s \equiv M(S) \pmod{\kappa+1}, \\
        \epsilon &  \text{ if }  \mathds{I}(S)=0  \text{ and } s\not\equiv M(S) \pmod{\kappa+1},
    \end{cases}
    \end{equation*}
    where $0<\epsilon<\frac{1}{\kappa}$ and 
    \begin{equation*}
    M(S):=
    \begin{cases} 
        \sum_{x\in \Pa{S}{\G'[\mathbf{S}]}}x & \text{, if } S\in \mathbf{S}\setminus \{S_0\}, \\
        u_0[0]+\sum_{x\in \Pa{S}{\G'[\mathbf{S}]}}x & \text{, if $S=S_0$ }.
    \end{cases}
    \end{equation*}
    Note that $M(S)$ is an integer number. This is because $\Pa{S}{\G'[\mathbf{S}]}\subseteq \U^{\mathbf{S}}$ and thus all terms in the above definition belong to $[0:\kappa]$. 
    \begin{lemma} \label{lemma: valid model}
        The SEM constructed above belongs to $\mathbb{M}^+(\G')$. 
    \end{lemma}

    \paragraph{Existence of realization $\textbf{v}_0$:} \label{sec: complete}
    Herein, we show that for the aforementioned conditional distributions, there exists a realization $\textbf{v}_0$ such that $\eta(\mathbf{v}_0)$ is linearly independent from the set of the vectors in $\mathbf{\Omega}$ (in Equation \eqref{eq: lin vectors}).  
    Consider the following subset of $\dom{U_0}=\{\gamma_1,...,\gamma_d\}$ with $\frac{\kappa+1}{2}$ elements:
    \begin{equation*}
        \mathbf{\Gamma} := \Big\{(2x,0,\cdots,0)\!:\: x\in [0:\frac{\kappa-1}{2}]\Big\}.
    \end{equation*}
    Recall that for $\mathbf{v}\in \dom{\V'}$ and $i\in[0:m]$, $\theta_i(\mathbf{v})$ and $\eta(\mathbf{v})$ are two vectors in $\mathbb{R}^d$ with $j$-th entry corresponds to $U_0=\gamma_j$. 
    Suppose that $\mathbf{\Gamma}=\{\gamma_{j_1},...,\gamma_{j_{\frac{\kappa+1}{2}}}\}$. Next result shows that in our constructed models, all entries of $\theta_i(\textbf{v})$ with indices in $\{j_1,...,j_{\frac{\kappa+1}{2}}\}$ are equal.
    \begin{lemma} \label{lem: theta index}
        For any $\mathbf{v} \in \dom{\V'}$ and $i\in[0:m]$,
        \begin{equation*}
            \theta_{i,j_1}(\mathbf{v}) = \theta_{i,j_2}(\mathbf{v}) = \cdots= \theta_{i,j_{\frac{\kappa+1}{2}}}(\mathbf{v}).
        \end{equation*}
    \end{lemma}
    An immediate consequence of this result is that any linear combination of the vectors in $\mathbf{\Omega}$ will have equal entries at the indices in $\{j_1,...,j_{\frac{\kappa+1}{2}}\}$.
    Next, we show there exists a realization $\textbf{v}_0$ for which $\eta(\mathbf{v}_0)$ does not follow this pattern and thus it is linearly independent of all vectors in $\mathbf{\Omega}$.
    \begin{lemma}\label{lem: eta index}
        There exists $0<\epsilon<\frac{1}{\kappa}$ for which there exists $\mathbf{v}_0 \in \dom{\V'}$ and $1\leq r <t\leq \frac{\kappa+1}{2}$ such that 
        \begin{equation*}
            \eta_{j_r}(\mathbf{v}_0) \neq \eta_{j_t}(\mathbf{v}_0). 
        \end{equation*}
    \end{lemma}
    Lemma \ref{lem: eta index} implies that there exist $\M_1$ and $\epsilon$ for which there exists $\mathbf{v}_0 \in \dom{\V'}$ such that $\eta(\mathbf{v}_0)$ is linearly independent from the set of vectors in $\mathbf{\Omega}$.
    As we discussed before, this completes our proof for Theorem \ref{thm: main}.

\section{Conclusion}
    We revisited the problem of general identifiability and showed that the positivity assumption of observational distributions is crucial for the soundness of do-calculus rules.
    This assumption was ignored in previous work. 
    We presented a novel algorithm for g-identifiability, which is provably sound and complete considering the positivity assumption.

\bibliography{bibliography}

\begin{thebibliography}{13}
\providecommand{\natexlab}[1]{#1}
\providecommand{\url}[1]{\texttt{#1}}
\expandafter\ifx\csname urlstyle\endcsname\relax
  \providecommand{\doi}[1]{doi: #1}\else
  \providecommand{\doi}{doi: \begingroup \urlstyle{rm}\Url}\fi

\bibitem[Bareinboim and Pearl(2012)]{bareinboim2012causal}
Elias Bareinboim and Judea Pearl.
\newblock Causal inference by surrogate experiments: Z-identifiability.
\newblock In \emph{Proceedings of the Twenty-Eighth Conference on Uncertainty
  in Artificial Intelligence}, page 113–120, Arlington, Virginia, USA, 2012.
  AUAI Press.

\bibitem[Bareinboim and Pearl(2014)]{bareinboim2014transportability}
Elias Bareinboim and Judea Pearl.
\newblock Transportability from multiple environments with limited experiments:
  Completeness results.
\newblock \emph{Advances in neural information processing systems}, 27, 2014.

\bibitem[Bareinboim and Tian(2015)]{bareinboim2015recovering}
Elias Bareinboim and Jin Tian.
\newblock Recovering causal effects from selection bias.
\newblock In \emph{Proceedings of the AAAI Conference on Artificial
  Intelligence}, volume~29, 2015.

\bibitem[Huang and Valtorta(2006)]{huang2006identifiability}
Yimin Huang and Marco Valtorta.
\newblock Identifiability in causal bayesian networks: A sound and complete
  algorithm.
\newblock In \emph{AAAI}, pages 1149--1154, 2006.

\bibitem[Huang and Valtorta(2008)]{huang2008completeness}
Yimin Huang and Marco Valtorta.
\newblock On the completeness of an identifiability algorithm for
  semi-markovian models.
\newblock \emph{Annals of Mathematics and Artificial Intelligence}, 54\penalty0
  (4):\penalty0 363--408, 2008.

\bibitem[Lee et~al.(2019)Lee, Correa, and Bareinboim]{lee2019general}
Sanghack Lee, Juan~D Correa, and Elias Bareinboim.
\newblock General identifiability with arbitrary surrogate experiments.
\newblock In \emph{Uncertainty in Artificial Intelligence}, pages 389--398.
  PMLR, 2019.

\bibitem[Mokhtarian et~al.(2022)Mokhtarian, Jamshidi, Etesami, and
  Kiyavash]{mokhtarian2022causal}
Ehsan Mokhtarian, Fateme Jamshidi, Jalal Etesami, and Negar Kiyavash.
\newblock Causal effect identification with context-specific independence
  relations of control variables.
\newblock In \emph{International Conference on Artificial Intelligence and
  Statistics}, pages 11237--11246. PMLR, 2022.

\bibitem[Pearl(1995)]{pearl1995causal}
Judea Pearl.
\newblock Causal diagrams for empirical research.
\newblock \emph{Biometrika}, 82\penalty0 (4):\penalty0 669--688, 1995.

\bibitem[Pearl(2009)]{pearl2009causality}
Judea Pearl.
\newblock \emph{Causality}.
\newblock Cambridge university press, 2009.

\bibitem[Shpitser and Pearl(2006{\natexlab{a}})]{shpitser2006identification}
Ilya Shpitser and Judea Pearl.
\newblock Identification of joint interventional distributions in recursive
  semi-markovian causal models.
\newblock In \emph{Proceedings of the National Conference on Artificial
  Intelligence}, volume~21, page 1219, 2006{\natexlab{a}}.

\bibitem[Shpitser and Pearl(2006{\natexlab{b}})]{shpitser2012identification}
Ilya Shpitser and Judea Pearl.
\newblock Identification of conditional interventional distributions.
\newblock \emph{Proceedings of the 22nd Conference on Uncertainty in Artificial
  Intelligence}, 2006{\natexlab{b}}.

\bibitem[Tian and Pearl(2003)]{tian2003ID}
Jin Tian and Judea Pearl.
\newblock On the identification of causal effects.
\newblock Technical report, Department of Computer Science, University of
  California, 2003.

\bibitem[Tikka et~al.(2021)Tikka, Hyttinen, and Karvanen]{tikka2019causal}
Santtu Tikka, Antti Hyttinen, and Juha Karvanen.
\newblock Causal effect identification from multiple incomplete data sources: A
  general search-based approach.
\newblock \emph{Journal of Statistical Software}, 2021.

\end{thebibliography}

\onecolumn
\appendix
\begin{center}
    \bfseries\Large Appendix
\end{center}

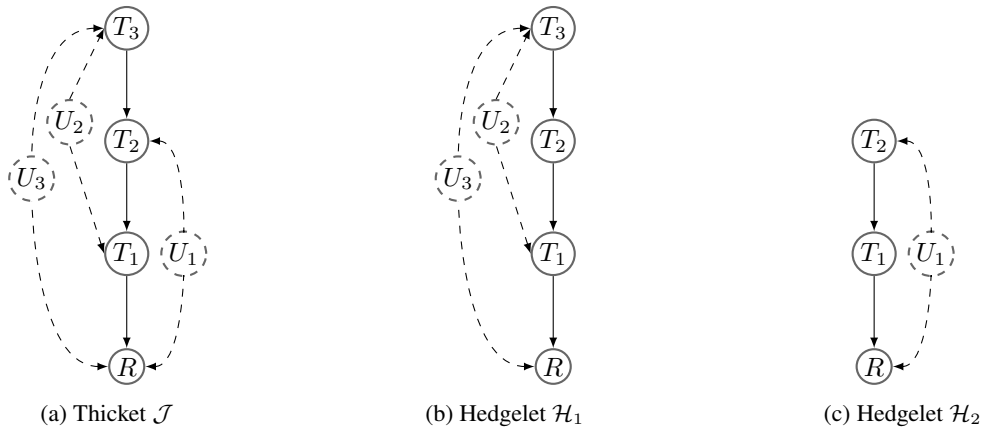
\begin{figure}[b]
    \centering
    \begin{subfigure}[b]{0.3\linewidth}
            \centering
            \begin{tikzpicture}[
            roundnode/.style={circle, draw=black!60,, fill=white, thick, inner sep=1pt},
            dashednode/.style = {circle, draw=black!60, dashed, fill=white, thick, inner sep=1pt},
            ]
            \node[roundnode]        (t1)        at (0, 0)                   {$T_1$};
            \node[roundnode]        (t2)        at (0, 1.5)                 {$T_2$};
            \node[roundnode]        (t3)        at (0, 3)                   {$T_3$};
            \node[roundnode]        (r)         at (0, -1.5)                {$R$};
            \node[dashednode]       (u1)        at (0.75, 0)                {$U_1$};
            \node[dashednode]       (u2)        at (-0.75, 1.75)            {$U_2$};
            \node[dashednode]       (u3)        at (-1.25, 1)               {$U_3$};
            
            \draw[-latex] (t1.south) -- (r.north) ;
            \draw[-latex] (t2.south) -- (t1.north);
            \draw[-latex] (t3.south) -- (t2.north);
            \draw[latex-, dashed] (t1.west) -- (u2.south);
            \draw[-latex, dashed] (u2.north) -- (t3.west);
            \draw[latex-, dashed] (r.east) .. controls +(right:5mm) and +(up:2mm) .. (u1.south);
            \draw[latex-, dashed] (t2.east) .. controls +(right:5mm) and +(down:2mm) .. (u1.north);
            \draw[latex-, dashed] (r.west) .. controls +(left:10mm) and +(up:2mm) .. (u3.south);
            \draw[latex-, dashed] (t3.west) .. controls +(left:10mm) and +(down:2mm) .. (u3.north);
                
            \end{tikzpicture}
            \caption{Thicket $\mathcal{J}$}
            \label{subfig: thicket exmple 2 }
    \end{subfigure}
    \begin{subfigure}[b]{0.3\linewidth}
            \centering
            \begin{tikzpicture}[
            roundnode/.style={circle, draw=black!60,, fill=white, thick, inner sep=1pt},
            dashednode/.style = {circle, draw=black!60, dashed, fill=white, thick, inner sep=1pt},
            ]
            \node[roundnode]        (t1)        at (0, 0)                   {$T_1$};
            \node[roundnode]        (t2)        at (0, 1.5)                 {$T_2$};
            \node[roundnode]        (t3)        at (0, 3)                   {$T_3$};
            \node[roundnode]        (r)         at (0, -1.5)                {$R$};
            \node[dashednode]       (u2)        at (-0.75, 1.75)            {$U_2$};
            \node[dashednode]       (u3)        at (-1.25, 1)               {$U_3$};
            
            \draw[-latex] (t1.south) -- (r.north) ;
            \draw[-latex] (t2.south) -- (t1.north);
            \draw[-latex] (t3.south) -- (t2.north);
            \draw[latex-, dashed] (t1.west) -- (u2.south);
            \draw[-latex, dashed] (u2.north) -- (t3.west);
            \draw[latex-, dashed] (r.west) .. controls +(left:10mm) and +(up:2mm) .. (u3.south);
            \draw[latex-, dashed] (t3.west) .. controls +(left:10mm) and +(down:2mm) .. (u3.north);
                
            \end{tikzpicture}
            \caption{Hedgelet $\mathcal{H}_1$}
            \label{subfig: hedgelet 1 exmpl 2}
    \end{subfigure}
    \begin{subfigure}[b]{0.3\linewidth}
            \centering
            \begin{tikzpicture}[
            roundnode/.style={circle, draw=black!60,, fill=white, thick, inner sep=1pt},
            dashednode/.style = {circle, draw=black!60, dashed, fill=white, thick, inner sep=1pt},
            ]
            \node[roundnode]        (t1)        at (0, 0)                   {$T_1$};
            \node[roundnode]        (t2)        at (0, 1.5)                 {$T_2$};
            \node[roundnode]        (r)         at (0, -1.5)                {$R$};
            \node[dashednode]       (u1)        at (0.75, 0)                {$U_1$};
            
            \draw[-latex] (t1.south) -- (r.north) ;
            \draw[-latex] (t2.south) -- (t1.north);
            \draw[latex-, dashed] (r.east) .. controls +(right:5mm) and +(up:2mm) .. (u1.south);
            \draw[latex-, dashed] (t2.east) .. controls +(right:5mm) and +(down:2mm) .. (u1.north);

            \end{tikzpicture}
            \caption{Hedgelet $\mathcal{H}_2$}
            \label{subfig: hedgelet 2 exmpl 2}
    \end{subfigure}
    \caption{(a) Thicket is formed for the causal effect of $\{T_1, T_2, T_3\}$ on $\{R\}$} in  Example 2; (b) and (c) are the hedgelets formed by the thicket $\mathcal{J}$
    \label{fig:exmpl 2}
\end{figure}

\section{On the positivity assumption} \label{sec: apd_pos}
    We first present some definitions and notations from \citep{lee2019general} including their  illustrations using the causal graph $\G$ from Example 2 (Figure \ref{fig:countr_exmpl_complex}).
    
    \subsection{Notation} \label{apn: A notation}
        \begin{definition}[\citep{lee2019general}]
            Assume that $\mathbf{R}$ is a subset of observed variables $\V$. A hedge is a pair of $\mathbf{R}$-rooted c-forests $\langle \F, \F' \rangle$ such that $\F'$ is a subgraph of $\F$.
        \end{definition}
        \textbf{In Figure \ref{fig:countr_exmpl_complex}:} Subgraphs $\F = \G[\{R, T_1, T_2, T_3\}]$ and $\F' = \G[\{R\}]$ form a hedge $\langle \F, \F' \rangle$.
        
        Denote by $\mathcal{C}(\G)=\{\textbf{W}_i\}_{i=1}^{k}$, the set of c-components that partition observed variables in $\G$ such that each $\W_i$ is a maximal c-component. Maximal in the sense of number of nodes that is there is no $\W \in \V$ such that $\W_i \subsetneq \W$ and $\W$ is a c-component in $\G$. 
        Assume that $\T$ is the set of all observed variables in $\F$ but not in $\F'$. 
        We define $\F'':=\F[\textbf{T}]$. 
        
        \textbf{In Figure \ref{fig:countr_exmpl_complex}:} $\mathcal{C}(\G[\{T_1, T_2, T_3\}]) = \{ \{T_1, T_3\}, \{T_2\}\}$. Additionally, $\F'' = \G[\{T_1, T_2, T_3\}]$ for the hedge constructed before.

        \begin{definition}[\cite{lee2019general}]
            Given a hedge $\langle \F, \F' \rangle$. Denote by $\V'$ a set of all observed variables of $\F'$. The hedgelet decomposition of a hedge $\langle \F, \F' \rangle$ is a collection of hedgelets $\{\F(\W)\}_{\W \in \mathcal{C}(\F'')}$ where each hedgelet $\F(\W)$ is a subgraph of $\F$ made of (i) $\F[\W\cup \V']$ and (ii) $\F[De_\mathcal{F}(\W)]$ without bidirected edges, that is all observed descendants of $\W$ and all directed edges between them. 
            Let $\mathbb{H}_{\F} := \{\F(\W)\}_{\W \in \mathcal{C}(\F'')}$ be the set of hedgelets of $\langle \F, \F' \rangle$.
        \end{definition}

        \textbf{In Figure \ref{fig:countr_exmpl_complex}:} For the hedge $\langle \F, \F' \rangle$, where $\F = \G[\{R, T_1, T_2, T_3\}]$ and $\F' = \G[\{R\}]$, there are two hedgelets $\mathcal{H}_1, \mathcal{H}_2$ displayed in Figures (\ref{subfig: hedgelet 1 exmpl 2})-(\ref{subfig: hedgelet 2 exmpl 2}). Moreover, we have $\mathbb{H}_{\F} = \{\mathcal{H}_1, \mathcal{H}_2\}$.
        
        \begin{definition}[\citep{lee2019general}]
            Let $\mathbf{R}$ be a non-empty set of variables and $\mathbb{Z}$ be a collection of sets of variables in $\G$. A thicket $\mathcal{J}$ is a subgraph of $\G$ which is an $\mathbf{R}$-rooted c-component consisting of a minimal c-component over $\mathbf{R}$ and hedges
            \begin{equation*}
                \mathbb{F}_{\mathcal{J}} := \{\langle \F_{\Z}, \mathcal{J}[\R]\rangle \mid \F_{\Z} \subseteq \G[\V\setminus\Z], \Z\cap \R=\varnothing\}_{\Z\in \mathbb{Z}}.
            \end{equation*}
        \end{definition}
        Let $\X$ and $\Y$ be disjoint sets of observed variables in $\G$. A thicket $\mathcal{J}$ is said to be formed for $P_{\mathbf{x}}(\mathbf{y})$ in $\G$ with respect to $\mathbb{Z}$ if $\R \subseteq \Anc{\Y}{\G[\V \setminus \X]}$ and every hedgelet of each hedge $\langle \F_{\Z}, \mathcal{J}[\R]\rangle$ intersects with $\X$.
        
        \textbf{In Figure \ref{fig:countr_exmpl_complex}:} This graph is a thicket, also displayed in Figure \ref{subfig: thicket exmple 2 }.  Let $\mathbb{F_{\mathcal{J}}}$ be
        \begin{equation*}
            \mathbb{F_{\mathcal{J}}} = \{\langle \F, \F'\rangle \},
        \end{equation*}
        where $\F = \G[\{R, T_1, T_2, T_3\}]$ and $\F' = \G[\{R\}]$.
        One can observe that thicket $\mathcal{J}$ is formed for the causal effect $\X = \{T_1, T_2, T_3\}$ on $\Y = \{R\}$.
        
        Denote by $\T$ all observed variables in thicket $\mathcal{J}$ outside of subgraph $\mathcal{J}[\R]$. Let $\mathbb{H} = \bigcup_{\{\langle \F, \F'\rangle \}\in \mathbb{F}_{\mathcal{J}}}\mathbb{H}_{\F}$, that is, a collection of all hedgelets induced by the hedges of $\mathcal{J}$.
        
        \textbf{In Figure \ref{fig:countr_exmpl_complex}:} $\T = \{T_1, T_2, T_3\}$ and $\mathbb{H} = \{\mathcal{H}_1, \mathcal{H}_2\}$.

    \subsection{On the positivity assumption} \label{apn: pos assumption simple}
        Given the above definitions, we can state Lemma 3 in \citep{lee2019general}.
        \begin{lemma*}
            Let $\T'\subsetneq \T$ such that there exists a hedgelet $\mathcal{H} \in \mathbb{H}\setminus\mathbb{H}(\T')$, where $\mathbb{H}(\T')$ is a set of hedgelets from $\mathbb{H}$ which contain at least one variable from $\textbf{T}'$. Then, under the intervention $do(\mathbf{t}')$, there exists $R \in \R$, for any instantiation of $\U$, such that $r=0$ in both models.
        \end{lemma*}

        Note that by the construction in \citep{lee2019general}, $R$ in the above Lemma is a binary random variable. 
        In the above Lemma, let $\T' = \varnothing$. 
        Based on this Lemma, for any instantiation of unobserved variables $\U$, $P(\V=\mathbf{v})=0$, where $\textbf{v}$ is a realization for observed variables in which $r=1$. 
        This clearly shows that the constructed models in \citep{lee2019general} violate the positivity assumption.

    \subsection{On the relaxed positivity assumption }

        Herein, we study Figure \ref{fig:countr_exmpl_complex} in more details and show that the models in \citep{lee2019general} violate the relaxed positivity assumption. 
        To this end, we present the models $\M_1$ and $\M_2$ constructed in \citep{lee2019general} for the thicket $\mathcal{J}$ which is defined for this case in Appendix \ref{apn: A notation}. 
        By the construction, each variable from $\{U_1, U_2, U_3, T_3\}$ is a binary number, i.e., $\{0, 1\}$ and each variable from $\{T_1, T_2\}$ is a vector of length two, because each variable from $\{U_1, U_2, U_3, T_3\}$ appears in only one hedgelet and each variables in $\{T_1, T_2\}$ appears in exactly two different hedgelets. 
        Thus, $T_1=(T_{1, 1}, T_{1, 2})$ and $T_2 = (T_{2, 1}, T_{2, 2})$, where $T_{1, 1}, T_{1, 2}, T_{2, 1}, T_{2, 2}$ are binary numbers. The first coordinate captures some properties of the hedgelet $\mathcal{H}_1$ while the second coordinate captures some properties of the hedgelet $\mathcal{H}_2$. 
        \cite{lee2019general} define both models $\M_1, \M_2$ for the hedgelet $\mathcal{H}_1$ as
        \begin{align*}
            & T_3 = U_2 \oplus U_3, \quad T_{2, 1} = T_3, \quad T_{1, 1} = T_{2, 1} \oplus U_2,
        \end{align*}
        and for the hedgelet $\mathcal{H}_2$ as
        \begin{align*}
            & T_{2, 2} = U_1, \quad T_{1, 2} = T_{2, 2}, \quad T_{2, 2} = U_1.
        \end{align*}
        Additionally, in model $\M_1$, variable $R$ is defined by
        \begin{equation*}
            R = \mathds{1}_{T_{1, 1} = 0} \wedge \mathds{1}_{T_{1, 2}=0} \wedge \mathds{1}_{U_3=1} \wedge \mathds{1}_{U_{1} = 1},
        \end{equation*}
        and in model $\M_2$, it is defined to be zero, i.e., $R = 0$.

\section{Technical proofs} \label{sec: apd_proof}

    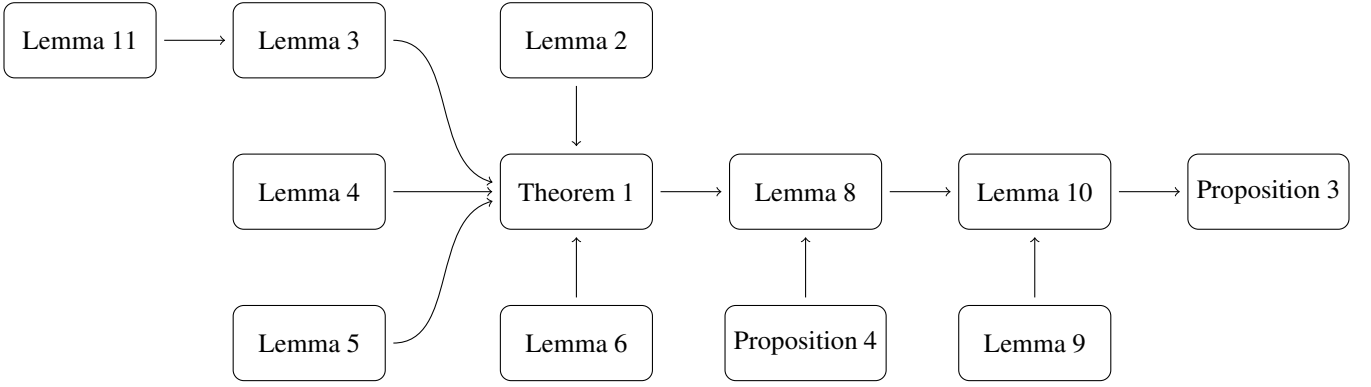
\begin{figure}[t]
        \centering
        \begin{tikzpicture}[block/.style={rounded corners, minimum width=2cm, minimum height=1cm, draw}]
            \node[block] (1) {Theorem \ref{thm: main}};
            \node[block, right=1 of 1] (3) {Lemma \ref{lemma: prp3 - 1}};
            \node[block, below=1 of 3] (2) {Proposition \ref{prp: 4}};
            \node[block, right=1 of 3] (4) {Lemma \ref{lemma: prp3 - 3}};
            \node[block, below=1 of 4] (5) {Lemma \ref{lemma: prp3 - 2}};
            \node[block, right=1 of 4] (6) {Proposition \ref{prp: 3}};
            \node[block, above=1 of 1] (7) {Lemma \ref{lem: simplify}};
            \node[block, left=1.5 of 1] (9) {Lemma \ref{lemma: valid model}};
            \node[block, above=1 of 9] (8) {Lemma \ref{lem: lin indep}};
            \node[block, below=1 of 9] (10) {Lemma \ref{lem: theta index}};
            \node[block, below=1 of 1] (11) {Lemma \ref{lem: eta index}};
            \node[block, left=1 of 8] (12) {Lemma \ref{lemma: lin indep formal}};
        \begin{scope}[->, shorten >=1mm, shorten <=1mm]
            \draw (1) -- (3);
            \draw (2) -- (3);
            \draw (3) -- (4);
            \draw (5) -- (4);
            \draw (4) -- (6);
            \draw (7) -- (1);
            \draw (8) to[out=0,in=165]([yshift=1mm]1.west);
            \draw (9) -- (1);
            \draw (10) to[out=0,in=195]([yshift=-1mm]1.west);
            \draw (11) -- (1);
            \draw (12) -- (8);
        \end{scope}
        \end{tikzpicture}
        \caption{Logical order of our proofs.}
        \label{fig: logic order}
    \end{figure}

    In this section, we first present some technical lemmas which we use throughout our proofs.
    The proofs of the lemmas and propositions within the main text are provided in Subsections \ref{sec: proof sec algorithm} and  \ref{sec: proof sec main}. 
    
    The logical order of our proofs is depicted in Figure \ref{fig: logic order}. For instance, we use Theorem \ref{thm: main} to prove Lemma \ref{lemma: prp3 - 1}. 
    Also note that the proof of Theorem \ref{thm: main} is provided in the main text using Lemmas \ref{lem: simplify}, \ref{lem: lin indep}, \ref{lemma: valid model}, \ref{lem: theta index}, and \ref{lem: eta index}. 
    
    \begin{definition}[Ancestral]
        We say a subset $\X$ of observed variables $\V$ is ancestral in $\G$, if $\X = \Anc{\X}{\G_{\V}}$.
    \end{definition}

    \subsection{Technical Lemmas}
        \begin{lemma}[\citep{tian2003ID}]\label{lemma: Q-marginal}
            Let $\W \subseteq \C \subseteq \V$, $\T = \C \setminus \W$, $\mathbf{S}=\V \setminus \T$. 
            If $\W$ is an ancestral set in $\G[\C]$, then:
            \begin{equation*}
                Q[\W] = \sum_{\C \setminus \W}Q[\C].
            \end{equation*}
        \end{lemma}
        
        \begin{lemma} \label{lemma: prp3 - 1}
        Consider a causal graph $\G$ with observed variables $\V$.
            Suppose $\X \subseteq \V$ and $e:=(X_1,Z)$ is a directed edge such that $X_1\in\X$.
            $Q[\X]$ is g-identifiable from $(\mathbb{A},\G)$ if and only if $Q[\X]$ is g-identifiable from $(\mathbb{A}, \mathcal{H})$, where $\mathcal{H}$ is the graph obtained by deleting $e$ from $\G$.
        \end{lemma}
        \begin{proof}
            $\X$ has the same c-components in $\G$ and $\mathcal{H}$ since $\G_{\V}$ and $\mathcal{H}_{\V}$ have the same undirected edges. 
            Let $\X_1,\cdots,\X_l$ be the c-components of $\X$. 
            For any $i\in [1:l]$ and $\A \in \mathbb{A}$ such that $\X_i \subseteq \A$, \cite{huang2008completeness} showed that $Q[\X_i]$ is identifiable from $\G[\A]$ if and only if $Q[\X_i]$ is identifiable from $\mathcal{H}[\A]$. 
            Hence, Theorem \ref{thm: main} implies that $Q[\X_i]$ is g-identifiable from $(\mathbb{A},\G)$ if and only if $Q[\X_i]$ is g-identifiable from $(\mathbb{A},\mathcal{H})$.
            In this case, Proposition \ref{prp: 4} implies that $Q[\X]$ is g-identifiable from $(\mathbb{A},\G)$ if and only if $Q[\X]$ is g-identifiable from $(\mathbb{A},\mathcal{H})$.
        \end{proof}
    
        \begin{lemma}\label{lemma: prp3 - 2}
            Suppose that $\X$ and $\Y$ are disjoint subsets of $\V$. Let $(Y_1,Y_2)$ (i.e., $Y_1\rightarrow Y_2$) denotes a directed edge in $\G$, where $Y_1,Y_2 \in \Y$. Let $\G'$ denotes the resulting graph after removing edge $(Y_1,Y_2)$ from $\G$. If the causal effect of $\X$ on $\Y$ is not g-identifiable from $(\mathbb{A}, \G')$, then the causal effect of $\X$ on $\Y \setminus \{Y_1\}$ is not g-identifiable from $(\mathbb{A}, \G)$.
        \end{lemma}
        \begin{proof}
            Herein, we provide a proof that is similar to one of the proofs in \citep{huang2008completeness}.
            
            Using Markov factorization property in graph $\G'$, $P_{\mathbf{x}}(\mathbf{y})$ is given by
            \begin{equation*}
                P_{\x}(\y) = 
                \sum_{\V\setminus(\X\cup \Y)} \sum_{\U} \prod_{W \in \V\setminus \X}P(w \mid \Pa{W}{\G'}) \prod_{U \in \U} P(u).
            \end{equation*}
            Similarly, in graph $\G$ we have
            \begin{equation*}
                P_{\x}(\y \setminus\{Y_1\}) = 
                \sum_{\{Y_1\} \cup (\V \setminus (\X \cup \Y))} \sum_{\U} \prod_{W \in \V \setminus \X} P(w \mid \Pa{W}{\G}) \prod_{U \in \U} P(u).
            \end{equation*}
            Since the causal effect of $\X$ on $\Y$ is not g-identifiable from $(\mathbb{A}, \G')$, there exists $\M_1$ and $\M_2$ in $\mathbb{M}^+(\G')$ such that:
            \begin{equation*}
                Q^{\M_1}[\A_i](\mathbf{v}) = Q^{\M_2}[\A_i](\mathbf{v}),\; \forall \mathbf{v}\in \dom{\V},\; \forall i \in [0: m],
            \end{equation*}
            \begin{equation*}
                P_{\x}^{\M_1}(\y)\neq P_{\x}^{\M_2}(\y),\; \exists \x \in \dom{\X}, \; \exists \y \in \dom{\Y}.
            \end{equation*}
            Using $\M_1$ and $\M_2$, we construct two SEMs $\M_1'$ and $\M_2'$ in $\mathbb{M}^+(\G)$. 
            Define a surjective function $F\!:\: \dom{Y_1}\rightarrow \{0, 1\}$ and a function $\Psi \!:\: \{0, 1\}\times \dom{Y_1} \rightarrow (0,1)$ such that $\Psi(0, y_1)+\Psi(1, y_1)=1$ for each $y_1 \in \dom{Y_1}$.
            We will later assume some constraints for these functions, but for now lets assume they are arbitrary. 
            
            For any node $S$ which is either unobserved or in $\V \setminus (\{Y_2\}\cup \Ch{Y_2}{\G})$, we define
            \begin{equation*}
                P^{\M_i'}(s|\Pa{S}{\G}) = P^{\M_i}(s|\Pa{S}{\G'}),
            \end{equation*}
            where $i \in \{1, 2\}$.
            The domain of $Y_2$ in $\M_i'$ is defined as $\dom{Y_2}^{\M}\times \{0, 1\}$, where $\dom{Y_2}^{\M}$ is the domain of $Y_2$ in $\M_i$. 
            For $y_2 \in \dom{Y_2}^\M$, $i\in \{0,1\}$, and $k\in \{0, 1\}$ we define:
            \begin{equation*}
                P^{\M_i'}((y_2, k) 
                \mid \Pa{Y_2}{\G'}, y_1) =
                P^{\M_i}(y_2 \mid \Pa{Y_2}{\G'}) \Psi(F(y_1)\oplus k, y_1).
            \end{equation*}
            Note that $\Pa{Y_2}{\G'}\cup \{Y_1\}= \Pa{Y_2}{\G}$.
            Moreover, for a fixed realization $(\Pa{Y_2}{\G'}, y_1)$, we have:
            \begin{equation*}
                \sum_{k\in \{0,1\}} \sum_{y_2\in \dom{Y_2}^{\M}} P^{\M_i'}((y_2, k)|pa(Y_2), y_1) = 1.
            \end{equation*}
            
            For each $S \in \Ch{Y_2}{\G}$, we define:
            \begin{equation*}
                P^{\M_i'}(s \mid \Pa{S}{\G}\setminus \{Y_2\}, (y_2, k)) =
                P^{\M_i}(s \mid \Pa{S}{\G}\setminus \{Y_2\}, y_2).
            \end{equation*}
            
            Next, we show that $Q^{\M_1'}[\A_i](\mathbf{v}) = Q^{\M_2'}[\A_i](\mathbf{v})$ for each $\mathbf{v}\in \dom{\V}$ and $i \in [0:m]$.
            Suppose $\mathbf{v}$ is a realization of $\V$ in $\M'_1$ with realizations $y_1$ and $(y_2, k)$ for $Y_1$ and $Y_2$, respectively. 
            Consider two cases: 
            \begin{itemize}
                \item If $Y_2 \notin \A_i$:
                \begin{align*}
                    Q^{\M_1'}[\A_i](\mathbf{v}) 
                    &= \sum_{\U}\prod_{A\in \A_i} P^{\M_1'}(a \mid \Pa{A}{\G})\prod_{U\in \U}P^{\M_1'}(u) \\
                    & = \sum_{\U} \prod_{A\in \A_i}P^{\M_1}(a \mid \Pa{A}{\G'})\prod_{U\in \U} P^{\M_1}(u) 
                    = Q^{\M_1}[\A_i](\mathbf{v}) 
                    = Q^{\M_2}[\A_i](\mathbf{v}) \\
                    & = \sum_{\U} \prod_{A\in \A_i} P^{\M_2}(a \mid \Pa{A}{\G'}) \prod_{U\in \U} P^{\M_2}(u) \\
                    &= \sum_{\U} \prod_{A\in \A_i} P^{\M_2'}(a \mid \Pa{A}{\G}) \prod_{U\in \U} P^{\M_2'}(u) \\ 
                    &= Q^{\M_2'}[\A_i](\mathbf{v}).
                \end{align*}
                \item If $Y_2 \in \mathbf{A}_i$:
                \begin{align*}
                    Q^{\M_1'}[\A_i](\mathbf{v}) 
                    &= \sum_{\U} \prod_{A\in \A_i} P^{\M_1'} (a \mid \Pa{A}{\G})\prod_{U\in \U} P^{\M_1'}(u) \\
                    & = \Psi\left(F(y_1)\oplus k, y_1\right) \sum_{\U} \prod_{A\in \A_i} P^{\M_1}(a \mid  \Pa{A}{\G'}) \prod_{U\in \U}P^{\M_1}(u) \\
                    &= \Psi(F(y_1)\oplus k, y_1) Q^{\M_1}[\A_i](\mathbf{v}) = \Psi(F(y_1)\oplus k, y_1) Q^{\M_2}[\A_i](\mathbf{v}) \\
                    &= \Psi(F(y_1)\oplus k, y_1) \sum_{\U} \prod_{A\in \A_i} P^{M_2}(a \mid \Pa{A}{\G'}) \prod_{U\in \U} P^{M_2}(u) \\
                    &= \sum_{\U}\prod_{A\in \A_i} P^{\M_2'}(a \mid \Pa{A}{\G})) \prod_{U\in \U}P^{\M_2'}(u) \\
                    &= Q^{\M_2'}[\A_i](\mathbf{v}).
                \end{align*}
            \end{itemize}
            Therefore, $Q^{\M_1'}[\A_i](\mathbf{v}) = Q^{\M_2'}[\A_i](\mathbf{v})$ for each $\mathbf{v}\in \dom{\V}$ and $i \in [0:m]$.
            
            We know that there exists $\hat{\x} \in \dom{\X}^{\M}$ and $\hat{\y} \in \dom{\Y}^{\M}$ such that $P^{\M_1}_{\hat{\x}}(\hat{\y})\neq P^{\M_2}_{\hat{\x}}(\hat{\y})$. 
            Denote by $\hat{y}_1$ and $\hat{y}_2$ the realizations of $Y_1$ and $Y_2$ in the realization $\hat{\y}$, respectively. 
            Assume that $P^{\M_1}_{\hat{\x}}(\hat{\y}) = d_1 > P^{\M_2}_{\hat{\x}}(\hat{\y}) = d_2$. 
            Assume that $\Psi(F(\hat{y}_1)\oplus 0, \hat{y}_1)=0.5$ and $\Psi(F(y)\oplus 0, y) = \frac{d_1-d_2}{4}$ for all $y \in \dom{Y_1}\setminus \{\hat{y}_1\}$. 
            Then we have:
            \begin{align*}
                P_{\hat{\x}}^{\M_1'}(\hat{\y} \setminus \{\hat{y}_1\})
                &=\sum_{y_1 \in \dom{Y_1}} \sum_{\V \setminus (\X\cup\Y)} \sum_{\U} \prod_{Z\in \V\setminus \X} P^{\M_1'}(z \mid \Pa{Z}{\G}) \prod_{U\in \U} P(u) \\
                & >\sum_{y_1 = \hat{y}_1} \sum_{\V\setminus (\X\cup\Y)} \sum_{\U} \prod_{Z\in \V\setminus \X} P^{\M_1'}(z \mid \Pa{Z}{\G}) \prod_{U\in \U} P(u) \\
                & =P_{\hat{\x}}^{\M_1}(\hat{\y}) \Psi(F(\hat{y}_1)\oplus 0, \hat{y}_1) = 0.5 d_1.
            \end{align*}
            but, 
            \begin{align*}
                P_{\hat{\x}}^{\M_2'}(\hat{\y} \setminus \{\hat{y}_1\})
                &=\sum_{y_1 \in \dom{Y_1}} \sum_{\V \setminus (\X \cup \Y)} \sum_{\U}  \prod_{Z\in \V \setminus \X} P^{\M_1'}(z \mid \Pa{Z}{\G}) \prod_{U\in \U} P(u)\\
                & =\sum_{y_1 = \hat{y}_1} \sum_{\V \setminus (\X \cup \Y )} \sum_{\U} \prod_{Z\in \V \setminus \X} P^{\M_1'}(z \mid \Pa{Z}{\G}) \prod_{U\in \U} P(u) \\
                & + \sum_{y_1 \in \dom{Y_1}\setminus \{\hat{y}_1\}} \sum_{\V \setminus (\X\cup\Y)} \sum_{\U} \prod_{Z\in \V \setminus \X} P^{M_1'}(z \mid \Pa{Z}{\G}) \prod_{U \in \U} P(u) \\
                &\leq P_{\hat{\x}}^{\M_2}(\hat{\y}) \Psi(F(\hat{y}_1)\oplus 0, \hat{y}_1) + P_{\hat{\x}}^{\M_2}(\hat{\y} \setminus \{\hat{y}_1\}) \Psi(F(Y_1\neq y_1)\oplus 0, Y_1\neq y_1)\\
                & = 0.5d_2 + \frac{d_1-d_2}{4} < 0.5d_1.
            \end{align*}
            This implies that $P_{\hat{\x}}^{\M_1'}(\hat{\y} \setminus \{\hat{y}_1\})\neq P_{\hat{\x}}^{\M_2'}(\hat{\y} \setminus \{\hat{y}_1\})$ which concludes the proof.
        \end{proof}
      
        \begin{lemma}\label{lemma: prp3 - 3}
            Assume $\Y \subset \W \subset \V$ such that for each $W \in \W \setminus \Y$, there exists a directed path in $\G[\W]$ from $W$ to a variable in $\Y$. 
            Then, the causal effect of $\V \setminus \W$ on $\Y$ is g-identifiable from $(\mathbb{A}, \G)$ if and only if $Q[\W]$ is g-identifiable from $(\mathbb{A}, \G)$. 
        \end{lemma}
        \begin{proof}
            Let $\X:=\V \setminus \W$. 
            
            \textit{Sufficient part}:
            Suppose $Q[\W]$ is g-identifiable from $(\mathbb{A}, \G)$.
            Since $Q[\W] = P_{\x}(\W)$, we have
            \begin{equation*}
                P_{\x}(\y) = \sum_{\W \setminus \Y} Q[\W].
            \end{equation*}
            Hence, $P_{\x}(\y)$ is uniquely computed and the causal effect of $\X$ on $\Y$ is g-identifiable from $(\mathbb{A}, \G)$. 
            
            \textit{Necessary part:} 
            Suppose $Q[\W]$ is not g-identifiable from $(\mathbb{A}, \G)$, we will show that $P_{\x}(\y)$ is also not g-identifiable. 
            To this end, first, we order the nodes in $\W \setminus \Y$, say $(W_1, W_2, \cdots, W_n)$, such that for each $1\leq i \leq n$, $W_i$ is a parent of at least one node in $\Y\cup \{W_1, W_2, \dots, W_{i-1}\}$. 
            Assume that $e_i$ is the directed edge from $W_i$ to its child in $\Y\cup \{W_1, W_2, \dots, W_{i-1}\}$.
            We also define $\G'$ to be the graph obtained by deleting all the edges $\{e_i\}_{i=1}^n$ from $\G$. 
            Applying Lemma \ref{lemma: prp3 - 1} repeatedly $n$ times imply that $Q[\W]$ is not g-identifiable from $(\mathbb{A},  \G')$.
            
            Let $\G_n:=\G$ and for $0\leq i\leq n-1$, we define $\G_i$ to be the graph obtained by removing $e_{i+1}$ from $\G_{i+1}$.
            From Lemma \ref{lemma: prp3 - 2}, we know that if $Q[\W]$ is not g-identifiable from $(\mathbb{A}, \G')$, then adding edge $e_1$ will make the causal effect of $\X$ on $\W \setminus \{W_1\}$ not g-identifiable from $(\mathbb{A}, \G_1)$.
            Note that $\G_1$ is obtained from $\G'$ by adding edge $e_1$. 
            Using this lemma again implies that the causal effect of $\X$ on $\W \setminus \{W_1, W_2\}$ is not g-identifiable from $(\mathbb{A}, \G_2)$. 
            Repeating this procedure yields that the causal effect of $\X$ on $\W \setminus \{W_1, \dots W_n\}=\Y$ is not g-identifiable from $(\mathbb{A}, \G_n)$. 
            Since $\G_n=\G$, the causal effect of $\X$ on $\Y$ is not g-identifiable from $(\mathbb{A}, \G)$. 
        \end{proof}
    
       \begin{lemma}\label{lemma: lin indep formal}
            Consider a set of vectors $\{c_i\}_{i=1}^{n}$, where $c_i \in \mathbb{R}^d$. Assume $c\in \mathbb{R}^d$ is a vector that is linearly independent of $\{c_i\}_{i=1}^{n}$, then there is a vector $b\in \mathbb{R}^d$ such that
            \begin{align*}
                & \langle c_i, b \rangle = 0, \quad \forall i \in [1:n],\\
                & \langle c, b \rangle \neq 0.
            \end{align*}
        \end{lemma}
        \begin{proof}
            Denote by $\{\phi_i\}_{i=1}^{l}$ a subset of $\{c_i\}_{i=1}^{n}$ which forms a basis for the vectors in $\{c_i\}_{i=1}^{n}$. Clearly, we have $l < d$. Now, consider the following system of linear equations with respect to $b$:
            \begin{equation} \label{eq: lem11}
            \begin{split}
                & \langle \phi_i, b \rangle = 0, \quad \forall i\in[1:l], \\
                & \langle c, b \rangle = 13\neq0.
            \end{split}
            \end{equation}
            By the assumption, vectors in $\{\phi_i\}_{i=1}^{l}\cup\{c\}$ are linearly independent, thus there exists a solution to \eqref{eq: lem11}.
        \end{proof}

    \subsection{Proofs of Section \ref{sec: algorithm}} \label{sec: proof sec algorithm}
        \begin{customprp}{\ref{prp: 3}}
            Let $\X$ and $\Y$ be two disjoint subsets of $\V$.
            The causal effect of $\X$ on $\Y$ is g-identifiable from $(\mathbb{A},\G)$ if and only if $Q[\Anc{\Y}{\G_{\V \setminus \X}}]$ is g-identifiable from $(\mathbb{A},\G)$.
        \end{customprp}
        \begin{myproof}[Proof]
            Let $\mathbf{W}:=\Anc{\Y}{\G_{\V \setminus \X}}$.
            Since $Q[\V \setminus \X] = P_{\x}(\V \setminus \X)$, using marginalization, we obtain
            \begin{equation}\label{eq:app_prp3}
                P_{\mathbf{x}}(\mathbf{y}) = \sum_{\V \setminus (\X \cup \Y)}Q[\V \setminus \X] = \sum_{\W \setminus \Y} \sum_{\V \setminus (\W \cup \X)} Q[\V \setminus \X].
            \end{equation}
            Since $\W$ is an ancestral set in $\G[\V \setminus \X]$, Lemma \ref{lemma: Q-marginal} implies
            \begin{equation*}
                \sum_{\V \setminus (\W \cup \X)} Q[\V \setminus \X] = Q[\W].
            \end{equation*}
            Substituting the above equation into \eqref{eq:app_prp3} implies
            \begin{equation} \label{eq: prp 3 proof}
                P_{\mathbf{x}}(\mathbf{y}) = \sum_{\W \setminus \Y} Q[\W] = P_{\mathbf{v}\setminus \mathbf{w}}(\y).
            \end{equation}
            \textit{Sufficient part:} 
            Suppose $Q[\W]$ is g-identifiable from $(\mathbb{A},\G)$.
            Equation \eqref{eq: prp 3 proof} implies that $P_{\mathbf{x}}(\mathbf{y})$ is uniquely computable from $Q[\W]$, and therefore, the causal effect of $\X$ on $\Y$ is g-identifiable from $(\mathbb{A},\G)$.
            
            \textit{Necessary part:}
            Suppose $Q[\W]$ is not g-identifiable from $(\mathbb{A}, \G)$. 
            For each $W \in \W \setminus \Y$, there exists a directed path in $\G[\W]$ from $W$ to a variable in $\Y$. 
            Hence, Lemma \ref{lemma: prp3 - 3} implies that the causal effect of $\V \setminus \W$ on $\Y$ is not g-identifiable from $(\mathbb{A}, \G)$. 
            Hence, Equation \eqref{eq: prp 3 proof} implies that $ P_{\mathbf{x}}(\mathbf{y})$ cannot be uniquely computed and the causal effect of $\X$ on $\Y$ is not g-identifiable from $(\mathbb{A},\G)$.
        \end{myproof}
        
        \begin{customprp}{\ref{prp: 4}}
            Suppose $\mathbf{S}\subseteq \V$ and $\mathbf{S}_1,\cdots,\mathbf{S}_l$ are the c-components of $\mathbf{S}$. 
            $Q[\mathbf{S}]$ is g-identifiable from $(\mathbb{A}, \G)$ if and only if $Q[\mathbf{S}_i]$ is g-identifiable from $(\mathbb{A}, \G)$ for each $i\in [1:l]$.
        \end{customprp}
        \begin{myproof}[Proof]
            \textit{Sufficient part}: 
            Suppose $Q[\mathbf{S}_i]$ is g-identifiable from $(\mathbb{A}, \G)$ for each $i\in [1:l]$.
            \cite{tian2003ID} showed that 
            \[Q[\mathbf{S}] = \prod_{i=1}^{l}Q[\mathbf{S}_i].\]
            Hence, $Q[\mathbf{S}]$ is uniquely computable and therefore, g-identifiable from $(\mathbb{A}, \G)$. 
            
            \textit{Necessary part}: 
            Suppose $Q[\mathbf{S}]$ is g-identifiable from $(\mathbb{A}, \G)$. 
            For $i\in [1:l]$, \cite{tian2003ID} provided a formula for computing $Q[\mathbf{S}_i]$ from $Q[\mathbf{S}]$ (Lemma 4, Equations (71) and (72) in \citep{tian2003ID}). 
            Hence, for each $i\in [1:l]$, $Q[\mathbf{S}]$ is uniquely computable and therefore, g-identifiable from $(\mathbb{A}, \G)$.
        \end{myproof}
    
\subsection{Proofs of Section \ref{sec: main}} \label{sec: proof sec main}
    \begin{customlem}{\ref{lem: simplify}}
        If $Q[\mathbf{S}]$ is not g-identifiable from $(\mathbb{A}', \G')$, then $Q[\mathbf{S}]$ is not g-identifiable from $(\mathbb{A}, \G)$.
    \end{customlem}
    \begin{proof}
        If $Q[\mathbf{S}]$ is not g-identifiable from $(\mathbb{A}', \G')$, then there exists two models $\M_1'$ and $\M_2'$ in $\mathbb{M}^+(\G')$ such that for each $i\in [0:m]$ and any $\mathbf{v} \in \dom{\V'}$,
        \begin{equation*}
            Q^{\M_1'}[\mathbf{A}_i'](\mathbf{v}) = Q^{\M_2'}[\mathbf{A}_i'](\mathbf{v}), 
        \end{equation*}
        and there exists $\mathbf{v}_0 \in \dom{\V'}$ such that 
        \begin{equation*}
            Q^{\M_1'}[\mathbf{S}](\mathbf{v}_0) \neq Q^{\M_2'}[\mathbf{S}](\mathbf{v}_0).
        \end{equation*}
        
        Next, we will construct two models $\M_1$ and $\M_2$ in $\mathbb{M}^+(\G)$ to prove that $Q[\mathbf{S}]$ is not g-identifiable from $(\mathbb{A}, \G)$. 
        We define the domains of variables in $\V'$ in the model $\M_i$ similar to model $\M_i'$, for $i \in \{1, 2\}$. 
        Since for each node $V\in\V'$, we have $\Pa{V}{\G'}\subseteq\Pa{V}{\G}$, then for all $V\in \V'$ and $i \in \{1, 2\}$, we can define:
        \begin{equation*}
            P^{M_i}(V|\Pa{V}{\G}) := P^{M_i'}(V \mid \Pa{V}{\G'}).
        \end{equation*}
        And for $V \in \V \setminus \V'$, we define:
        \begin{equation*}
            \dom{V} = \{0\}, \quad P(V=0)=1.
        \end{equation*}
        Because variable $V \in \V \setminus \V'$ can only take value $0$ with probability one, then $Q^{\M_j}[\mathbf{A}_i](\mathbf{v})=Q^{\M'_j}[\mathbf{A}'_i](\mathbf{v})$ for all $i$ and $Q^{\M_j}[\mathbf{S}](\mathbf{v}_0)=Q^{\M_j'}[\mathbf{S}](\mathbf{v}_0)$ for $j\in\{1,2\}$. 
        Thus, we have
        \begin{align*}
            &Q^{\M_1}[\mathbf{A}_i](\mathbf{v}) = Q^{\M_2}[\mathbf{A}_i](\mathbf{v}), \quad i\in[0:m],\\
            &Q^{\M_1}[\mathbf{S}](\mathbf{v}_0) \neq Q^{\M_2}[\mathbf{S}](\mathbf{v}_0).
        \end{align*}
        This shows that $Q[\mathbf{S}]$ is not g-identifiable from $(\mathbb{A}, \G)$.
    \end{proof}
 
    \begin{customlem}{\ref{lem: lin indep}}
        Consider the following set of vectors in $\mathbb{R}^d$
         \begin{equation} \label{eq: lin vectors apd}
            \mathbf{\Omega}:=\{\theta_{i}(\mathbf{v}):\ i\in [0:m], \mathbf{v}\in \dom{\V}\} \cup \mathds{1}_d,
        \end{equation}
        where $\mathds{1}_d$ denotes the all-ones vector in $\mathbb{R}^d$. 
        If there exists $\mathbf{v}_0 \in \dom{\V}$ such that $\eta(\mathbf{v}_0)$ is linearly independent from all the vectors in $\mathbf{\Omega}$, then the system of linear equations in \eqref{eq: linear system 2} admits a solution.
    \end{customlem}
    \begin{myproof}[Proof]
        This is a direct consequence of Lemma \ref{lemma: lin indep formal} with $\{c_i\}$ to be $\mathbf{\Omega}$ and $c$ to be $\eta(\mathbf{v}_0)$.
    \end{myproof}
    
    \begin{customlem}{\ref{lemma: valid model}}
        The SEM constructed above belongs to $\mathbb{M}^+(\G')$. 
    \end{customlem}
    \begin{proof}
        By the construction, it is clear that the model belongs to $\mathbb{M}(\G')$. 
        Hence, we need to show that $P(\mathbf{v})>0$ for any $\mathbf{v}\in \dom{\V'}$. 
        To this end, it is enough to show that for any realization $\mathbf{v} \in \dom{\V}'$, there exists a realization $\hat{\mathbf{u}}\in \dom{\U'}$ such that $P(\mathbf{v}, \hat{\mathbf{u}})>0$, because in this case we have
        \begin{equation*}
            P(\mathbf{v}) = \sum_{\mathbf{u}=\dom{\U'}} P(\mathbf{v,u}) \geq P(\mathbf{v}, \hat{\mathbf{u}}) > 0.
        \end{equation*}
        Let $\mathbf{v}$ be a fixed realization in $\dom{\V'}$. 
        For the rest of the proof, we assume all the realizations for $\V'$ are consistent with $\mathbf{v}$.
        
        By Markov factorization property,  for any $\mathbf{u}\in \dom{\U'}$ we have 
        \begin{equation}
            P(\mathbf{v}, \mathbf{u}) = \prod_{V\in \V'}P(v\mid \Pa{V}{\G'})\prod_{U\in \U'}P(u).
        \end{equation}
        By the construction of our model, we have $P(u)>0$ for any $U\in \U'$ and $u\in \dom{U}$. 
        Moreover, for any $X\in \mathbf{S}$ and any realization for $\Pa{X}{\G'}\cap \U'$ we have $P(x \mid \Pa{X}{\G'})>0$. 
        Hence, it is enough to show that there exists $\hat{\mathbf{u}}\in \dom{\U'}$ such that $P(x \mid \Pa{X}{\G'})>0$ for each $X\in \T$.
        
        Recall that for each $X\in \T$, we have $X=(X[i_1],\cdots, X[i_{\alpha(X)}])$, where $X$ belongs to $\F_{i_1}, \cdots, \F_{i_{\alpha(X)}}$ and 
        \begin{equation*}
            X[i_j] \equiv \left(\sum_{Y \in \Pa{X}{\F_{i_j}}} Y[i_j]\right) \pmod{2}.
        \end{equation*}
        
        By the construction, we define the entries corresponding to each $\F_i$ separately. 
        For each $i\in [0:k]$, let $\U_i$ to be the set of unobserved variables in $\U^{\T}$ that are in $\F_i$. 
        
        Let us fix an $i\in [0:k]$. 
        To finish the prove, we will introduce a method to determine $\hat{u}[i]$ for each $U\in \U_i$ such that 
        \begin{equation} \label{eq: eq for T}
            x[i] \equiv \left(\sum_{Y \in \Pa{X}{\F_{i}}} y[i]\right) \pmod{2},
        \end{equation}
        for each $X\in \T \cap \B_i$. 
        
        Lets start with an arbitrary set of values for $\{\hat{u}[i]\!:\: U\in \U_i\}$ which are either $0$ or $1$. 
        Suppose $X\in \T \cap \B_i$. 
        We introduce a trick such that $x[i]$ will be replaced by $1- x[i]$ while for all $Y\in \T \cap \B_i$, $y[i]$ remains the same: \\
        By the construction of $\F_i$, there exists a path $(X=X_1, U_1, X_2, \cdots, X_l,U_l,Z = X_{l+1} )$ from $X$ to a variable $Z \in \mathbf{S}$ such that $\{U_1, \cdots, U_l\} \subseteq \U_i$, $\{X_1,\cdots,X_l\} \subseteq \B_i \cap \T$, and $\Ch{U_j}{\F_i} = \{X_j, X_{j+1}\}$ for each $j\in [1:l]$. 
        Now for each $j \in [1:l]$, we replace $\hat{u}_j[i]$ by $1-\hat{u}_j[i]$. 
        Since Equation \eqref{eq: eq for T} is in mod $2$, the value of $x_j[i]$ will be the same for each $j\in [2:l]$ while $x[i]$ will be replaced by $1-x[i]$. 
        Note that $X_{l+1}=Z \notin \T$. 
        
        With the trick described above, we can construct any realization for the $i$-th bit of the variables in $\T \cap \B_i$. 
        Hence, we can construct $\hat{\mathbf{u}}\in \dom{\U'}$ such that $P(x \mid \Pa{X}{\G'})>0$ for each $X\in \T$.
    \end{proof}
    
   \begin{customlem}{\ref{lem: theta index}}
        For any $\mathbf{v} \in \dom{\V'}$ and $i\in[0:m]$,
        \begin{equation*}
            \theta_{i,j_1}(\mathbf{v}) = \theta_{i,j_2}(\mathbf{v}) = \cdots= \theta_{i,j_{\frac{\kappa+1}{2}}}(\mathbf{v}).
        \end{equation*}
    \end{customlem}
    \begin{proof}
        Lets fix a realization $\mathbf{v}$ for the observed variables $\V'$.
        Suppose that $l_1$ and $l_2$ are two integers such that
        \begin{equation*}
        \begin{split}
            & \gamma_{l_1} = (2x, 0, \dots, 0)),\\
            & \gamma_{l_2} = (2x+2 \pmod{\kappa+1}, 0, \dots, 0),
        \end{split}
        \end{equation*}
        where $x$ is any fixed integer in $[0 : \frac{\kappa-1}{2}]$.
        To show the result, we will prove that $\theta_{i, l_1}(\mathbf{v})=\theta_{i, l_2}(\mathbf{v})$. Let
        \begin{align*}
            &f_{i, j}(\mathbf{v}, \mathbf{u}^{\mathbf{T}}) := \sum_{\mathbf{u}\in\U^{\mathbf{S}}} \prod_{V \in \A'_i} P(v \mid \Pa{V}{\G'})\! \prod_{U\in \U' \setminus \{U_0\}}\! P(u)\\
            & = \prod_{V \in \A'_i\setminus\B_i} P(v \mid \Pa{V}{\G'})\! \prod_{V \in \B_i\setminus\mathbf{S}} P(v \mid \Pa{V}{\G'})\!\sum_{\mathbf{u}\in\U^{\mathbf{S}}}\prod_{V \in \mathbf{S}} P(v \mid \Pa{V}{\G'}) \prod_{U\in \U' \setminus \{U_0\}}\! P(u). 
        \end{align*}
        where index $j$ indicates $U_0=\gamma_j$. Note that variable $U_0$ may appear in the parent set of some observed variables. 
        Using the above definition, we have
        \begin{equation*}
            \theta_{i, j}(\mathbf{v}) = \sum_{\mathbf{u}^\textbf{T}\in\U^{\T}} f_{i, j}(\mathbf{v}, \mathbf{u}^{\mathbf{T}}).
        \end{equation*}
        Hence, if we show $f_{i, l_1}(\mathbf{v}, \mathbf{u}^{\mathbf{T}})=f_{i, l_2}(\mathbf{v}, \mathbf{u}^{\mathbf{T}})$ for any fixed realization $\mathbf{u}^{\T}$, the above equation implies $\theta_{i, l_1}(\mathbf{v})=\theta_{i, l_2}(\mathbf{v})$. 
        
        When $T\in \A'_i\setminus\B_i$, then for fixed realizations of $\textbf{u}^{\T}$,  $P(t|\Pa{T}{\G'})$ is the same for both realizations $\gamma_{l_1}$ and $\gamma_{l_2}$ since $\gamma_{l_1}\equiv\gamma_{l_2}$ mod 2.
        
        When $T\in \B_i\setminus\mathbf{S}$, unobserved variables in $\Pa{T}{\G'}$ are a subset of $\U^{\T}\cup\{U_0\}$. Note that in the definition of $f_{i, j}(\mathbf{v}, \mathbf{u}^{\mathbf{T}})$, all such unobserved variables are fixed. 
        Thus, if there exists $T\in \B_i\setminus\mathbf{S}$, such that $P(t|\Pa{T}{\G'})=0$, then 
        \begin{equation*}
        f_{i, l_1}(\mathbf{v}, \mathbf{u}^{\mathbf{T}})=f_{i, l_2}(\mathbf{v}, \mathbf{u}^{\mathbf{T}})=0.
        \end{equation*}
        When $P(t|\Pa{T}{\G'})=1$ for all $T\in \B_i\setminus\mathbf{S}$,
        to prove $f_{i, l_1}(\mathbf{v}, \mathbf{u}^{\mathbf{T}})=f_{i, l_2}(\mathbf{v}, \mathbf{u}^{\mathbf{S}})$, we show that for any realization $(\textbf{u}_1,\gamma_{l_1})$ of $(\U^{\mathbf{S}},U_0)$, there is a realization $(\textbf{u}_2,\gamma_{l_2})$  of $(\U^{\mathbf{S}},U_0)$ such that 
        \begin{equation*}
            \prod_{V \in \mathbf{S}} P(v \mid \Pa{V}{\G'})\Big|_{(\U^{\mathbf{S}},U_0)=(\textbf{u}_1,\gamma_{l_1})}=\prod_{V \in \mathbf{S}} P(v \mid \Pa{V}{\G'})\Big|_{(\U^{\mathbf{S}},U_0)=(\textbf{u}_2,\gamma_{l_2})},
        \end{equation*}
        where $P(v \mid \Pa{V}{\G'})\Big|_{(\U^{\mathbf{S}},U_0)=(\textbf{u}_1,\gamma_{l_1})}$ denotes the conditional probability of $v$ given its parents in which the unobserved variables $(\U^{\mathbf{S}},U_0)$ are fixed to be  $(\textbf{u}_1,\gamma_{l_1})$.
        To this end, we consider two cases depending on $i$. 
        
        \textbf{First case, when $i \in [0:k]$:} 
        In this case, we have
        \begin{equation}\label{eq: tmp1}
            t[i] = \left( \sum_{Y \in \Pa{T}{\F_i}} y[i] \right) \pmod{2}.
        \end{equation}
        
        Consider the set $\mathbf{\Lambda}:=\Pa{\mathbf{S}}{\F_i}\setminus \Pa{\mathbf{S}}{\F_i[\mathbf{S}]}$, that is the set of all parents of nodes in $\textbf{S}$ that are outside of $\textbf{S}$. By the construction of our models, summation of the values of the observed and unobserved nodes in $\mathbf{\Lambda}$ are the same, i.e.,  
        \begin{equation*}
            \sum_{W \in \mathbf{\Lambda}\cap \B_i} w[i] \equiv \sum_{W \in \mathbf{\Lambda}\cap \U'} w[i] \pmod{2},
        \end{equation*}
        or equivalently
        \begin{equation}\label{eq: tmp2}
            \sum_{W \in \mathbf{\Lambda}} w[i] \equiv 0 \pmod{2}.
        \end{equation}
        This is because, in graph $\F_i$, each observed variable outside of $\mathbf{S}$ has at most one child outside of $\mathbf{S}$, and each unobserved node has either one or two children outside of $\mathbf{S}$. 
        According to (\ref{eq: tmp1}), those unobserved nodes with two children outside of $\mathbf{S}$ do not belong to $\mathbf{\Lambda}\cap \U'$. 
        Such unobserved nodes have exactly two observed descendants in $\mathbf{\Lambda}\cap \B_i$, and because both descendants appear in \eqref{eq: tmp2}, their summation is zero mod 2. 
        On the other hand, the unobserved nodes with only one child outside of $\mathbf{S}$ belong to $\mathbf{\Lambda}\cap \U'$ and have exactly one observed descendant in $\mathbf{\Lambda}\cap \B_i$. Thus, the summation of such unobserved variables and their observed descendant is again zero mod 2 in \eqref{eq: tmp2}.
 
        If $\mathbb{I}(S)=0$ for all $S \in \mathbf{S}$, then by our model construction, for any variable $W \in \mathbf{\Lambda}\setminus \{T_i\}$, $w[i]$ is an even number but $T_i$ takes value 1 with probability one. 
        Hence, the summation in \eqref{eq: tmp2} cannot be an even number.
        Therefore, there exists at least a variable $S\in \mathbf{S}$ such that $\mathbb{I}(S)=1$. 
        In this case, the value of $P(S|\Pa{S}{\G'})$ does not depend on the realizations of variables in $\U^{\mathbf{S}}$. 
        Next, we show that for any realization $\mathbf{u}_{1}$ of $\U^{\mathbf{S}}$, there is a realization $\mathbf{u}_{2}$ such that 
        \begin{equation}\label{eq:app_prop_3}
        P(s|\Pa{S}{\G'})\Big|_{(\U^{\mathbf{S}},U_0)=(\textbf{u}_1,\gamma_{l_1})}=P(s|\Pa{S}{\G'})\Big|_{(\U^{\mathbf{S}},U_0)=(\textbf{u}_2,\gamma_{l_2})}.
        \end{equation}
        Since $\G'_{\mathbf{S}}$ is a c-component, there exists a sequence of variables $U_0, \hat{S}_1, \hat{U}_1, \hat{S}_2, \hat{U}_2, \dots, \hat{U}_l, S$, such that $U_0$ is a parent of $\hat{S}_1$, $S$ is a children of $\hat{U}_l$ and $\hat{U}_j$ is a parent of variables $\hat{S}_j$ and $\hat{S}_{j+1}$ for $j \in [1:l-1]$. 
        Let $\hat{\U}:=\{\hat{U}_1, \dots, \hat{U}_l\}$. 
        For realization $\mathbf{u}_{1}$, we define $\mathbf{u}_{2}$ by
        \begin{equation}
            \begin{split}
                & u_{2,\hat{U}_j} := u_{1,\hat{U}_j} + 2(-1)^{j} \pmod{\kappa+1}, \quad j\in[1:l],\\
                & u_{2,U} := u_{1,U}, \quad \forall U\in \U' \setminus (\hat{\U}\cup \{U_0\}),
            \end{split}
        \end{equation}
        where $u_{2,U}$ denotes the realization for variable $U$ in $\textbf{u}_2$.
        It is straightforward to see that this mapping is a bijection between $\mathbf{u}_{1}$ and $\mathbf{u}_{2}$ and \eqref{eq:app_prop_3} holds.
        
        \textbf{Second case, when $i \in [k+1:m]$:} 
        In this case, $\mathbf{S} \setminus \A_i'\neq \varnothing$.  
        Since $\G'_{\mathbf{S}}$ is a c-component, there exists a sequence of variables $U_0, \hat{S}_1, \hat{U}_1, \hat{S}_2, \hat{U}_2, \dots, \hat{U}_l, S$, such that $U_0$ is a parent of $\hat{S}_1$, $S\in \mathbf{S}\setminus \A_i'$ is a children of $\hat{U}_l$ and $\hat{U}_j$ is a parent of variables $\hat{S}_j$ and  $\hat{S}_{j+1}$ for $j \in [1:l-1]$. 
        Let $\hat{\U}:=\{\hat{U}_1, \dots, \hat{U}_l\}$. 
        Similar to the previous case, for a given realization $\mathbf{u}_{1}$ of $\U^{\mathbf{S}}$,  we define $\mathbf{u}_{2} \in \dom{\U^{\mathbf{S}}}$ by
        \begin{equation}
            \begin{split}
                & u_{2,\hat{U}_j} := u_{1,\hat{U}_j} + 2(-1)^{j} \pmod{\kappa+1}, \quad j\in[1:l],\\
                & u_{2,U} := u_{1,U}, \quad \forall U\in \U' \setminus (\hat{\U}\cup \{U_0\}),
            \end{split}
        \end{equation}
        where $u_{2,\hat{U}}$ denotes the realization for variable $U$ in $\textbf{u}_2$. Analogous to the previous setting, we have \eqref{eq:app_prop_3}.
        
        Herein, we proved that $\theta_{i, l_1}(\mathbf{v})=\theta_{i, l_2}(\mathbf{v})$. By varying $x$ within $[0 : \frac{\kappa-1}{2}]$ in the definition of $\gamma_{l_1}$ and $\gamma_{l_2}$, we  conclude the lemma.
    \end{proof}

   \begin{customlem}{\ref{lem: eta index}}
        There exists $0<\epsilon<\frac{1}{\kappa}$ such that there exists $\mathbf{v}_0 \in \dom{\V'}$ and $1\leq r <t\leq \frac{\kappa+1}{2}$ such that 
        \begin{equation*}
            \eta_{j_r}(\mathbf{v}_0) \neq \eta_{j_t}(\mathbf{v}_0).
        \end{equation*}
    \end{customlem}
    \begin{proof}
        Lets consider $r$ and $t$ such that $\gamma_r = (0, 0, \dots, 0)$ and $\gamma_t = (2, 0, \dots, 0)$. Recall that:
        \begin{align}
            \label{eq: tmp4.1}
            & \eta_r(\mathbf{v}) :=\! \sum_{\U'\setminus \{U_0\}} \prod_{X \in \mathbf{S}} P(x \mid \Pa{X}{\G'})\Big|_{U_0=\gamma_r}\! \prod_{U\in \U'\setminus \{U_0\}}\! P(u), \\
            \label{eq: tmp4.2}
            & \eta_t(\mathbf{v}) :=\! \sum_{\U'\setminus \{U_0\}} \prod_{X \in \mathbf{S}} P(x \mid \Pa{X}{\G'})\Big|_{U_0=\gamma_t}\! \prod_{U\in \U'\setminus \{U_0\}}\! P(u).
        \end{align}
        
        We choose $\textbf{v}_0$ as follows: set all variables in $\mathbf{S}$ to be zero and select a realization for variables in $\V' \setminus \mathbf{S}$ such that $\mathbb{I}(S)=0$ for all $S \in \mathbf{S}$. 
        Denote by $S_0$ a child of $U_0$ in $\mathbf{S}$.
        
        Note that there is a term in the summation of the right side of equation (\ref{eq: tmp4.1}) that is $(1-\kappa\epsilon)^{|\mathbf{S}|}$. For instance, this occurs when all realizations of unobserved variables in $\U^{\mathbf{S}}$ are zero.
         
        Next, we prove that there is no realization of unobserved variables $\U^{\mathbf{S}}$ such that $P(S|\Pa{S}{\G'})=1-\epsilon\kappa$ for all $S\in \mathbf{S}$ and $U_0=\gamma_t$.
        In other words, each term in the summation of \eqref{eq: tmp4.2} has at least a term $\epsilon$. 
        To do so, it suffices to show that there is no realization of $\U^{\mathbf{S}}$ such that:
        \begin{equation*}
            \begin{split}
                & s = \sum_{W \in \Pa{S}{\G'[\mathbf{S}]}} w, \quad S\in \mathbf{S}\setminus \{S_0\},\\
                & s_0 = u_0[0] + \sum_{W \in \Pa{S}{\G'[\mathbf{S}]}} w.
            \end{split}
        \end{equation*}
        Suppose there is a realization of $\U^{\mathbf{S}}$ such that the above equations hold. In this case, since $\G'_{\mathbf{S}}$ is a tree, we can color its nodes with two colors, red and black, such that connected nodes by biderected edges have different colors. 
        Suppose that $\mathbf{S}_1$ is the set of black variables and $\mathbf{S}_2$ is the set of red variables which (without loss of generality) contains $S_0 \in \mathbf{S}_1$.
        Then:
        \begin{equation*}
            \begin{split}
                & \sum_{W\in \mathbf{S}_1} w \equiv u_0[0] + \sum_{U\in \U^{\mathbf{S}}} u \pmod{\kappa+1},\\
                & \sum_{W\in \mathbf{S}_2} w \equiv  \sum_{U\in \U^{\mathbf{S}}} u \pmod{\kappa+1}.
            \end{split}
        \end{equation*}
        The left-hand sides of both above equations are zero because of our choice of $\textbf{v}_0$.
        However, the right-hand sides cannot be the same since $u_0[0]=2$.
        Hence, in Equation \eqref{eq: tmp4.2}, there exists a term in the summation with probability $\epsilon$.
        Therefore, in extreme case, when $\epsilon=0$, $\eta_{t}(\mathbf{v}')=0$. However, 
        $\eta_{r}(\mathbf{v}')\geq (1-\kappa\epsilon)^{|\mathbf{S}|}\prod_{U\in \U'\setminus \{U_0\}}\! P(u)>0$. 
        Since $\eta_{r}(v)$ and $\eta_{t}(v)$ are polynomial functions of $\epsilon$ and they are not equal at $\epsilon=0$, then there exists a small enough $0<\epsilon<\frac{1}{\kappa}$ such that $\eta_{r}(\mathbf{v}') \neq \eta_{t}(\mathbf{v}')$.
    \end{proof}
    
\section{A special case in the proof of Theorem \ref{thm: main}} \label{sec: apd_second case}
    In this section, we provide our proof for the necessary part of Theorem \ref{thm: main} when $\mathbf{S}\nsubseteq \A_i'$ for all $i\in [0:m]$. 
     
    We define $\F^{\mathbf{S}}$ to be a minimal (in terms of edges) spanning subgraph of $\G[\mathbf{S}]$ such that $\F^{\mathbf{S}}_{\mathbf{S}}$ is a single c-component. 
    In this case, we can assume $\V' = \mathbf{S}$, $\G'$ is $\F^{\mathbf{S}}$, and $\mathbb{A}'= \{\A_i':=\A_i \cap \V'\}_{i=0}^m$. 
    For each $i\in [0:m]$, we have $\A_i' \subsetneq \V'$. 
    Note that Lemma \ref{lem: simplify} holds for this case. 
    Hence, it is enough to show that $Q[\mathbf{S}]$ is not g-identifiable from $(\mathbb{A}', \G')$. 
    
    Recall that our assumptions and goal in this section are as follows:\\
    $\G'$ is a DAG with observed variables $\V'$ and unobserved variables $\U'$ such that $\G'_{\V'}$ has no directed edges and its bidirected edges form a spanning tree over $\V'$. 
    $\mathbb{A}'= \{\A_i'\}_{i=0}^m$ is a collection of subsets such that $\A_i' \subsetneq \V'$. 
    The goal is to show that $Q[\V']$ is not g-identifiable from $(\mathbb{A}',\G')$.

    For this case we will define two model $\M_1$ and $\M_2$ such that for each $i\in [0:m]$ and any $\mathbf{v}\in \dom{\V'}$,
    \begin{equation*}
        Q^{\M_1}[\A_i'](\mathbf{v}) = Q^{\M_2}[\A_i'](\mathbf{v}),
    \end{equation*}
    but there exists $\mathbf{v}_0\in \dom{\V'}$ such that
    \begin{equation*}
        Q^{\M_1}[\mathbf{S}](\mathbf{v}_0) \neq Q^{\M_2}[\mathbf{S}](\mathbf{v}_0).
    \end{equation*}
    
    For both models $\M_1$ and $\M_2$ we define each observed and unobserved variable to be binary, i.e $\dom{W}=\{0, 1\}$ for all $W \in \V' \cup \U'$.
    Next, we define the equation of the variables in each model. 
    
    \textbf{Model 1}: For $V \in \V'$:
    \begin{align}
        V = 
        \begin{cases}
            \bigoplus \Pa{V}{\G'}, \quad \text{with probability } 1-\epsilon,\\
            1, \quad \text{with probability } \frac{\epsilon}{2},\\
            0, \quad \text{with probability } \frac{\epsilon}{2},
        \end{cases}
    \end{align}
    and for $U \in \U'$:
    \begin{equation*}
        P(U=0)= P(U=1) = 0.5.
    \end{equation*}
    
    \textbf{Model 2}: Suppose $V_1$ is a fixed observed variable in $\V'$. 
    Then, for all $V$ in $\V'\setminus \{V_1\}$ we define:
    \begin{align}
        V = 
        \begin{cases}
            \bigoplus \Pa{V}{\G'}, \quad \text{with probability } 1-\epsilon \\
            1, \quad \text{with probability } \frac{\epsilon}{2},\\
            0, \quad \text{with probability } \frac{\epsilon}{2},
        \end{cases}
    \end{align}
    and for $V_1$:
    \begin{align}
        V_1 = 
        \begin{cases}
            \urcorner\bigoplus \Pa{V_1}{\G'}, \quad \text{with probability } 1-\epsilon \\
            1, \quad \text{with probability } \frac{\epsilon}{2},\\
            0, \quad \text{with probability } \frac{\epsilon}{2},
        \end{cases}
    \end{align}
    where $\urcorner$ denotes the logical not. 
    Similar to the first mode, for each unobserved variables $U\in\U'$, 
    \begin{equation*}
        P(U=0)= P(U=1) = 0.5.
    \end{equation*}
    
    \begin{lemma}
        Let $i \in [0, m]$ and denote the cardinality of $\A_i'$ by $n$, i.e. $|\A_i'|=n$. 
        Then for any realization $\mathbf{v}\in \dom{\V'}$:
        $$
        Q^{\M_1}[\A_i'](\mathbf{v}) = Q^{\M_2}[\A_i'](\mathbf{v}) = \frac{1}{2^n}.
        $$
    \end{lemma}
    \begin{proof}
        Suppose $\A_i':= \{A_1, A_2, \dots, A_n\}$. Since $\A_i' \subsetneq \V'$, there are distinct unobserved variables $U_1, U_2, \dots, U_n$, such that $U_j$ is a parent of the $A_j$ for $j \in [1: n]$. Denote by $\M$ any of the model $\M_1$ or $\M_2$. 
        
        Assume that for some realization of observed and unobserved variables, exactly $t\in [0, n]$ variables in $\A'_i$ are defined by the $XOR$ or $\urcorner XOR$ of their parents. Without loss of generality, assume that these variables are $\{A_1, A_2, \dots, A_t\}$. If we know all unobserved variables $\U'$ except $\{U_1, U_2, \dots, U_t\}$, then we can determine uniquely the values of $\{U_1, U_2, \dots, U_t\}$ from the following equations:
        \begin{equation*}
            A_i = \bigotimes \Pa{A_i}{\G'},\quad i\in [1:t],
        \end{equation*}
        where $\bigotimes$ denotes the corresponding equation, either $XOR$ or $\urcorner XOR$, for variable $A_i$ in model $\M$. 
        Thus, by considering all possible realizations of unobserved variables that lead to a realization $\mathbf{v} \in \dom{\V'}$, we obtain
        \begin{equation*}
            Q[\textbf{A}_i'] = \sum_{j=0}^n C_n^j(1-\epsilon)^j\left( \frac{\epsilon}{2} \right)^{n-j} \left(\frac{1}{2}\right)^j = \left(\frac{1}{2}\right)^n,
        \end{equation*}
        where $C_n^j$ is the number of different ways to choose $j$ variables out of $n$, such that with probability $(1-\epsilon)$ their values are determined by either $XOR$ or $\urcorner XOR$ equation. All other $n-j$ variables are equal to either $0$ or $1$ with probability $\frac{\epsilon}{2}$.
    \end{proof} 
    
    \begin{lemma}
        Let $\mathbf{v}=\mathbf{0}$ be the realization of $\V'$ such that all observed variables are equal to $0$. 
        Then $Q^{\M_1}[\V'](\mathbf{v}) \neq Q^{\M_2}[\V'](\mathbf{v})$.
    \end{lemma}
    \begin{proof}
        Define  $n=|\V'|$ and $\V' = \{V_1, V_2, \dots, V_n\}$. Firstly, we will prove that for any $\mathbf{v} \in \dom{\V'}$, the value of $Q^{\M_2}[\V'](\mathbf{v})$ does not depend on the position of $V_1$ in graph $\G'$. 
        Denote by $V_2$ an observed variable which is connected to the $V_1$ by a bidirected edge in $\G'_{\V'}$.
        Let $U$ denotes the unobserved variable (corresponding to the bidirected edge) which is a parent of $V_1$ and $V_2$. 
        Next, we define a new model $\M_2'$ in which all variables in $\V'$ are defined similarly as they are defined in model $\M_2$ except for variables $V_1$ and $V_2$. 
        In $\M_2'$, we define $V_2$ in the same way as $V_1$ is defined in  $\M_2$.
        We also define $V_1$ in $\M_2'$ in the same way as $V_2$ is defined in $\M_2$.
        Then, we have
        \begin{align*}
            \prod_{i=1}^n P^{\M_2}(v_i|\Pa{V_i}{\G'}) = P^{\M_2}(v_1|\Pa{V_1}{\G'})P^{\M_2}(v_2|\Pa{V_2}{\G'})\prod_{i=3}^{n}P(v_i|\Pa{V_i}{\G'}) \\
            = P^{\M_2'}(v_1|\Pa{V_1}{\G'}\setminus\{U\}, u\oplus 1)P^{\M_2'}(v_2|\Pa{V_2}{\G'}\setminus\{U\}, u\oplus 1)\prod_{i=3}^{n}P(v_i|\Pa{V_i}{\G'}).
        \end{align*}
        This implies that substituting $V_1$ by $V_2$ does not change the value of $Q^{\M_2}[\V'](\mathbf{v})$.
        
        Without loss of generality, suppose that $V_1$ is a leaf in $\G'$ and $U_1$ is a parent of $V_1$. Note that there are exactly $n-1$ unobserved variables in graph $\G'$.
        This is because $\G'_{\V'}$ is a tree with bidirected edges over $\V'$. Therefore, we have
        \begin{align*}
            & 2^{n-1}Q^{\M_1}[\V'](\mathbf{0}) = P^{\M_1}(V_1=0|U_1=0)\!\! \sum_{\U' \setminus \{ U_1\}} \prod_{j>1}P(v_j|\Pa{V_j}{\G'}) + P^{\M_1}(V_1=0|U_1=1)\!\! \sum_{\U' \setminus \{ U_1\}} \prod_{j>1}P(v_j|\Pa{V_j}{\G'}), \\
            & 2^{n-1}Q^{\M_2}[\V'](\mathbf{0}) = P^{\M_2}(V_1=0|U_1=0)\!\! \sum_{\U' \setminus \{ U_1\}} \prod_{j>1}P(v_j|\Pa{V_j}{\G'}) + P^{\M_2}(V_1=0|U_1=1)\!\! \sum_{\U' \setminus \{ U_1\}} \prod_{j>1}P(v_j|\Pa{V_j}{\G'}).
        \end{align*}
        Note that:
        \begin{align*}
            & P^{\M_1}(V_1=0|U_1=0) = 1 -\frac{\epsilon}{2}\\
            & P^{\M_1}(V_1=0|U_1=1) = \frac{\epsilon}{2}\\
            & P^{\M_2}(V_1=0|U_1=0) = \frac{\epsilon}{2}\\
            & P^{\M_2}(V_1=0|U_1=1) = 1 -\frac{\epsilon}{2}
        \end{align*}
        More over, we have
        \begin{equation*}
            \sum_{U_1=0, \U' \setminus \{ U_1\}} \prod_{j>1}P(v_j|\Pa{V_j}{\G'}) + \sum_{U_1=1, \U' \setminus \{ U_1\}} \prod_{j>1}P(v_j|\Pa{V_j}{\G'}) = Q[\V'\setminus\{V_1\}] = \left(\frac{1}{2}\right)^{n-1}
        \end{equation*}
        This yields
        \begin{align*}
            & 2^{n-1}Q^{\M_1}[\V'](\mathbf{0}) = \left( 1 - \frac{\epsilon}{2} \right)a + \frac{\epsilon}{2} b, \\
            & 2^{n-1}Q^{\M_2}[\V'](\mathbf{0}) = \left( 1 - \frac{\epsilon}{2} \right)b + \frac{\epsilon}{2} a,
        \end{align*}
        where 
        \begin{align*}
            & a = \sum_{U_1=0, \U' \setminus \{ U_1\}} \prod_{j>1}P(V_j=0|\Pa{V_j}{\G'}), \\
            & b = \sum_{U_1=1, \U' \setminus \{ U_1\}} \prod_{j>1}P(V_j=0|\Pa{V_j}{\G'}).
        \end{align*}
        To prove that $Q^{\M_1}[\V'](\mathbf{0}) \neq Q^{\M_2}[\V'](\mathbf{0})$, it is enough to show that $a\neq b$.
        
        Denote by $S_n$ an observed variable connected to the $V_1$ by a bidirect edge in $\G'_{\mathbf{\V'}}$.
        We define $\V'_{n-1} := \V'\setminus \{V_1\}$, $\U'_{n-1} := \U'\setminus \{U_1\}$ and $\G_{n-1}:=\G'[\V' \setminus \{V_1\}]$.
        We also define models $\M_1^{(n-1)}$ and $\M_2^{(n-1)}$ as follows:
        
        \textbf{New model $\M_1^{(n-1)}$}: For $V \in \V_{n-1}'$:
        \begin{align}
            V = 
            \begin{cases}
                \bigoplus \Pa{V}{\G_{n-1}}, \quad \text{with probability } 1-\epsilon,\\
                1, \quad \text{with probability } \frac{\epsilon}{2},\\
                0, \quad \text{with probability } \frac{\epsilon}{2},
            \end{cases}
        \end{align}
        and for $U \in \U'_{n-1}$:
        \begin{equation*}
            P(U=0)= P(U=1) = 0.5.
        \end{equation*}
        
        \textbf{Model $\M_2^{(n-1)}$}:
        For all $V$ in $\V'_{n-1}\setminus \{S_{n}\}$:
        \begin{align}
            V = 
            \begin{cases}
                \bigoplus \Pa{V}{\G_{n-1}}, \quad \text{with probability } 1-\epsilon \\
                1, \quad \text{with probability } \frac{\epsilon}{2},\\
                0, \quad \text{with probability } \frac{\epsilon}{2},
            \end{cases}
        \end{align}
        and for $S_{n}$:
        \begin{align}
            S_n = 
            \begin{cases}
                \urcorner\bigoplus \Pa{S_{n}}{\G_{n-1}}, \quad \text{with probability } 1-\epsilon \\
                1, \quad \text{with probability } \frac{\epsilon}{2},\\
                0, \quad \text{with probability } \frac{\epsilon}{2}.
            \end{cases}
        \end{align}
        Similar to the first model, for each unobserved variables $U\in\U'_{n-1}$, we define
        \begin{equation*}
            P(U=0)= P(U=1) = 0.5.
        \end{equation*}
        Note that:
        \begin{align*}
            & \left( \frac{1}{2} \right)^{n-2}\sum_{U_1=0, \U'_{n-1}} \prod_{j>1}P(V_j|\Pa{V_j}{\G'}) =   \left( \frac{1}{2} \right)^{n-2}a = Q^{\M_1^{(n-1)}}[\V'_{n-1}](\mathbf{0}),\\
            & \left( \frac{1}{2} \right)^{n-2}\sum_{U_1=1, \U'_{n-1}} \prod_{j>1}P(V_j|\Pa{V_j}{\G'}) =   \left( \frac{1}{2} \right)^{n-2}b =   Q^{\M_2^{(n-1)}}[\V'_{n-1}](\mathbf{0}).
        \end{align*}
        It remains to show  $Q^{\M_1^{(n-1)}}[\V'_{n-1}](\mathbf{0})\neq Q^{\M_2^{(n-1)}}[\V'_{n-1}](\mathbf{0})$. 
        Note that if this holds, then by our construction, $Q^{\M_1}[\V'](\mathbf{0})\neq Q^{\M_2}[\V'](\mathbf{0})$.
        In other words, we could reduce the size of the graph while keeping the same problem. Thus, by continuing this procedure, we eventually reach graph $\G_2$ that consists of only two observed nodes and showing $Q^{\M_1^{(2)}}[\V'_{2}](\mathbf{0})\neq Q^{\M_2^{(2)}}[\V'_{2}](\mathbf{0})$ in that graph will conclude the result.
        For graph $\G_2$, we have
        \begin{align*}
            & Q^{\M_1^{(2)}}[\V'_{2}](\mathbf{0}) =\left( \frac{\epsilon}{2}\right)^2 + 2\frac{\epsilon}{2}(1-\epsilon)\frac{1}{2}+(1-\epsilon)^2\frac{1}{2}, \\
            & Q^{\M_2^{(2)}}[\V'_{2}](\mathbf{0}) = \left( \frac{\epsilon}{2}\right)^2 + 2\frac{\epsilon}{2}(1-\epsilon)\frac{1}{2}.
        \end{align*}
        This clearly shows that $Q^{\M_1^{(2)}}[\V'_{2}](\mathbf{0})\neq Q^{\M_2^{(2)}}[\V'_{2}](\mathbf{0})$.
    \end{proof}

\end{document}